\documentclass{article}

\usepackage{iclr2022_conference,times}

\usepackage{amsmath,amsfonts,bm}

\def\figref#1{figure~\ref{#1}}
\def\Figref#1{Figure~\ref{#1}}
\def\twofigref#1#2{figures \ref{#1} and \ref{#2}}
\def\trifigref#1#2#3{figures \ref{#1}, \ref{#2}, and \ref{#3}}

\def\secref#1{section~\ref{#1}}

\def\eqref#1{equation~\ref{#1}}

\providecommand{\defref}[1]{definition~\ref{#1}}

\providecommand{\thmref}[1]{theorem~\ref{#1}}
\providecommand{\Thmref}[1]{Theorem~\ref{#1}}
\providecommand{\lemref}[1]{lemma~\ref{#1}}
\providecommand{\Lemref}[1]{Lemma~\ref{#1}}
\providecommand{\corref}[1]{corollary~\ref{#1}}

\providecommand{\exampleref}[1]{example~\ref{#1}}

\providecommand{\remref}[1]{remark~\ref{#1}}

\def\1{\bm{1}}

\def\eps{{\epsilon}}

\DeclareMathAlphabet{\mathsfit}{\encodingdefault}{\sfdefault}{m}{sl}
\SetMathAlphabet{\mathsfit}{bold}{\encodingdefault}{\sfdefault}{bx}{n}

\def\gN{{\mathcal{N}}}

\def\sN{{\mathbb{N}}}

\def\sP{{\mathbb{P}}}

\def\sR{{\mathbb{R}}}

\def\sZ{{\mathbb{Z}}}

\providecommand{\Or}{\bar{r}}
\providecommand{\oX}{\bar{X}}
\providecommand{\oY}{\bar{Y}}

\newcommand{\E}[1]{\mathbb{E} \left[ #1 \right]}

\newcommand{\CE}[2]{\mathbb{E} \left[ #1 \middle| #2 \right]}

\providecommand{\inner}[2]{ \langle #1, #2 \rangle}  

\newcommand{\po }[1]{ \ensuremath{\lv #1 \rv_{\psi_1}} }  
\providecommand{\pt }[1]{ \ensuremath{\lv #1 \rv_{\psi_2}} }  
\providecommand{\lt }[1]{ \ensuremath{\lv #1 \rv} }  
\providecommand{\norm}[1]{\ensuremath{\left\| #1 \right\|}}

\makeatletter
\newcommand{\biggg}{\bBigg@{3}}
\newcommand{\vast}{\bBigg@{4}}
\newcommand{\Vast}{\bBigg@{5}}
\makeatother

\providecommand{\brho}{\bar{\rho}}
\renewcommand{\varrho}{\rho}

\usepackage[utf8]{inputenc} 
\usepackage[T1]{fontenc}    
\usepackage{hyperref}       
\usepackage{url}            
\usepackage{longtable}      
\usepackage{amsfonts}       
\usepackage{nicefrac}       
\usepackage{microtype}      
\usepackage{xcolor}         

\usepackage{amssymb}

\usepackage{amsthm}
\newtheorem{theorem}{Theorem}

\newtheorem{lemma}{Lemma}

\newtheorem{corollary}{Corollary}

\theoremstyle{definition}
\newtheorem{definition}{Definition}
\newtheorem{example}{Example}

\theoremstyle{definition}
\newtheorem{remark}{Remark}
\newtheorem*{remark*}{Remark}

\usepackage{textcomp}

\usepackage{hyperref}

\usepackage{mathtools}
\mathtoolsset{showonlyrefs=true}

\usepackage{graphicx}
\usepackage{color}
\usepackage{subcaption}
\usepackage{algorithm}
\usepackage{algorithmic}
\usepackage{tikz}
\usetikzlibrary{arrows.meta}
\usepackage{footnote}
\makesavenoteenv{tabular}

\usepackage{bbm}

\newcommand{\RR}{\mathbb{R}}

  \newcommand{\rv}{\right\Vert}
 \newcommand{\lv}{\left\Vert}

\DeclareMathOperator*{\Var}{Var}

\DeclareMathOperator{\ReLU}{R}
\newcommand{\R}{\widehat{{\text{R}}}}

\newcommand{\dualactthree}{\widehat{\sigma}}
\newcommand{\dualactone}{\widehat{\sigma}}
\newcommand{\Rone}{\widehat{\text{R}}}

\newcommand{\IN}{\mathrm{in}}
\newcommand{\OUT}{\mathrm{out}}

\providecommand{\IN}[1]{{\color{magenta} IN: #1}}

\title{A Johnson--Lindenstrauss Framework for \\ Randomly Initialized CNNs}

\author{Ido Nachum, Jan Hązła, Michael Gastpar
    \\ School of Computer and Communication Sciences
    \\  \'Ecole Polytechnique F\'ed\'erale de Lausanne 
    \\ 1015 Lausanne, Switzerland 
    \\ \texttt{$\langle$forename.surname$\rangle$@epfl.ch} 
\And
    Anatoly Khina
    \\ School of Electrical Engineering
    \\ Tel Aviv University
    \\ Tel Aviv 6997801, Israel
    \\ \texttt{anatolyk@eng.tau.ac.il} 
} 
\iclrfinalcopy 

\begin{document}

\maketitle


\begin{abstract}

    How does the geometric representation of a dataset change after the application of each randomly initialized layer of a neural network?
       The celebrated Johnson--Lindenstrauss lemma answers this question for linear fully-connected neural networks (FNNs), stating that 
      the geometry is essentially preserved.
      For FNNs with the ReLU activation, the angle between two inputs contracts according to a known mapping.
      The question for non-linear convolutional neural networks (CNNs) becomes much more intricate. 
      To answer this question, we introduce a geometric framework.
      For linear CNNs, we show that the Johnson--Lindenstrauss lemma continues to hold, namely, that the angle between two inputs is preserved. 
      For CNNs with ReLU activation, on the other hand, the behavior is richer:
      The angle between the outputs contracts, where the level of contraction  depends on the nature of the inputs.
      In particular, after one layer, the geometry of natural images is essentially preserved, 
      whereas for Gaussian correlated inputs, CNNs exhibit the same contracting behavior as FNNs with ReLU activation. 
\end{abstract}


\section{Introduction}

Neural networks have become a standard tool in multiple scientific fields, due to their success in classification (and estimation) tasks.
Conceptually, this success is achieved since better representation is allowed by each subsequent layer until linear separability is achieved at the last (linear) layer.
Indeed, in many disciplines involving real-world tasks, such as computer vision and natural language processing, the training process is biased toward these favorable representations. This bias is a product of several factors, with the neural-network initialization playing a pivotal role~\citep{Sutskever-Martens-Dahl-Hinton:Initialization-important:ICML2013}. 
Therefore, we concentrate in this work on studying the initialization of neural networks, 
with the following question guiding this work.

\emph{
    How does the geometric representation of a dataset change after the application of each randomly initialized layer of a neural network?
}

To answer this, we study how the following two geometrical quantities change after each layer.

\begin{center}  \label{similarity}
\begin{tabular}{c l}
     $\inner{x}{y}$    &   The scalar \textit{inner product} between vectors 
     $x$ and $y$.\footnote{We mean here a vector in a wider sense: $x$ and $y$ may be matrices and tensors (of the same dimensions). In this situation, the standard inner product is equal to the vectorization thereof: $\inner{x}{y} = \inner{\mathrm{vec}(x)}{\mathrm{vec}(y)}.$}
  \\ $\rho := \frac{\inner{x}{y}}{\norm{x}\norm{y}}$ & The \textit{cosine similarity} (or simply \textit{similarity)} between vectors $x$ and $y$.
\end{tabular}
\end{center}
 

The similarity $\rho \in [-1, 1]$ between $x$ and $y$ equals $\rho = \cos(\theta)$, where $\theta$ is the angle between them.


Consider first one layer of a fully connected neural network (FNN) with an identity activation (linear FNN) that is initialized by independent identically distributed (i.i.d.) Gaussian weights with mean zero and variance $1/N$, where $N$ is the number of neurons in that layer. 
This random linear FNN induces an isometric embedding of the dataset, namely,
the similarity $\rho_\IN$ between any two inputs, $x_\IN$ and $y_\IN$, is preserved together with their norm: 
\begin{align}
\label{eq:isometry:similarity}
    \varrho_\OUT \approx \brho_\OUT = \rho_\IN,
\end{align} 
where $\varrho_\OUT$ is the similarity between the resulting random (due to the multiplication by the random weights) outputs, $X_\OUT$ and $Y_\OUT$ (respectively), and $\brho_\OUT$ is the mean output similarity defined by 

\begin{align} 
\\[-2.5\baselineskip]
\label{eq:rho_out:def}
    \varrho_\OUT &:= \frac{\inner{X_\OUT}{Y_\OUT}}{\norm{X_\OUT}\norm{Y_\OUT}} ,
    &\text{and }&
  & \brho_\OUT &:= \frac{\E{ \inner{X_\OUT}{Y_\OUT} }}{\sqrt{\E{ \norm{X_\OUT}^2} \E{\norm{Y_\OUT}^2} }} .
\\[-1.4\baselineskip]
\end{align} 
%
%
The proof of this isometric relation between the input and output similarities follows from the celebrated Johnson--Lindenstrauss lemma \citep{Johnson-Lindenstrauss-Lemma:original1984,Johnson-Lindenstrauss-Lemma:Dasgupta-Gupta:probabilistic-proof:1999}.
This lemma states that a random 
linear map 
of dimension $N$ preserves the distance between any two points  
up to an $\eps > 0$ contraction/expansion with probability at least $1 - \delta$ for all $N > c \log(1/\delta) / \eps^2$ for an absolute constant $c$.
In the context of randomly initialized linear FNNs, this result means that,
for a number of neurons $N$ that satisfies $N > c \log(1/\delta) / \eps^2$, 
\begin{align}
    \sP \left( \left| \varrho_\OUT - \rho_\IN \right| < \eps \right) \geq 1 - \delta
\\[-1.4\baselineskip]
\end{align}

So, conceptually,  the Johnson--Lindenstrauss lemma studies how inner products (or geometry) change, in expectation,  after applying a random transformation and how well an average of these random transformations is  concentrated around the expectation. This is the exact setting of randomly initialized neural networks. The random transformations consist of a random projection (multiplying the dataset by a random matrix) which is followed by a non-linearity. 

Naturally, adding a non-linearity complicates the picture.
Let us focus on the case where the activation function is a rectified linear unit (ReLU).
That is, consider a random fully-connected layer with ReLU activation initialized with i.i.d. zero-mean Gaussian weights and two different inputs.
For this case, 
\citet{deep_ker}, \citet{giryes2016deep}, and \citet{daniely} proved that\footnote{The results in \citep{deep_ker} and \citep{daniely} were derived assuming unit-norm vectors $\norm{x_\IN} = \norm{y_\IN} = 1$. The result here follows by the homogeneity of the ReLU activation function: $\ReLU(\alpha x) = \alpha \ReLU(x)$ for $\alpha \geq 0$, and ergodicity, assuming multiple filters are applied.}
\begin{align} 
\label{eq:rho_in-rho_out}
    \varrho_\OUT \approx \brho_\OUT &= 
    \frac{\sqrt{1-\rho_\IN^2} +\left( \pi -\cos^{-1}(\rho_\IN) \right)\rho_\IN }{\pi} .
\\[-1.4\baselineskip]
\end{align} 

\begin{figure}\vspace{-1.5\baselineskip}[h]
 \begin{subfigure}{0.45\textwidth}
        \includegraphics[width=\linewidth]{  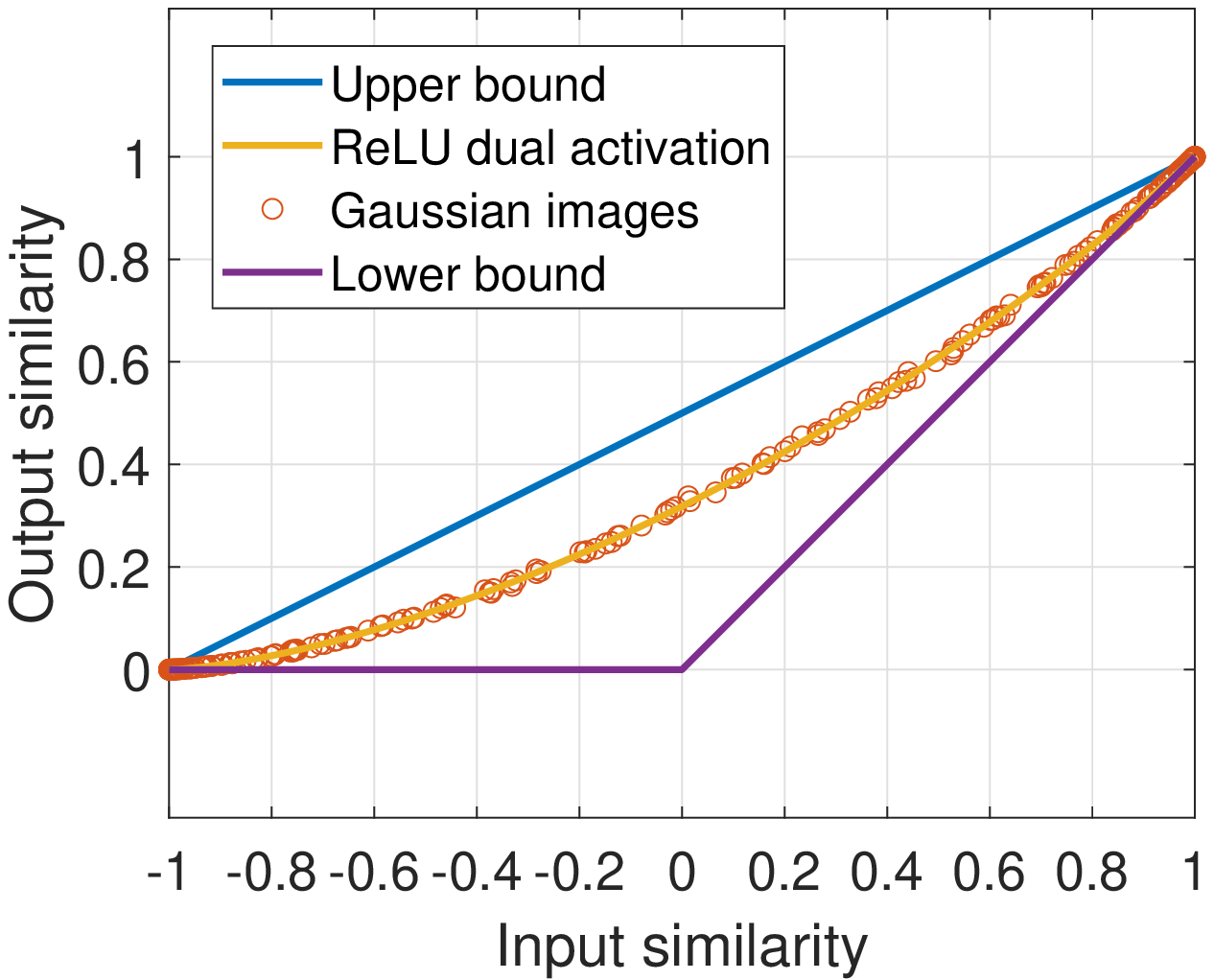}
        \caption{Gaussian images, filter size $3\times 3$}
    \label{fig:gaussian}
    \end{subfigure}\hfill
  \begin{subfigure}{0.45\textwidth}
        \includegraphics[width=\linewidth]{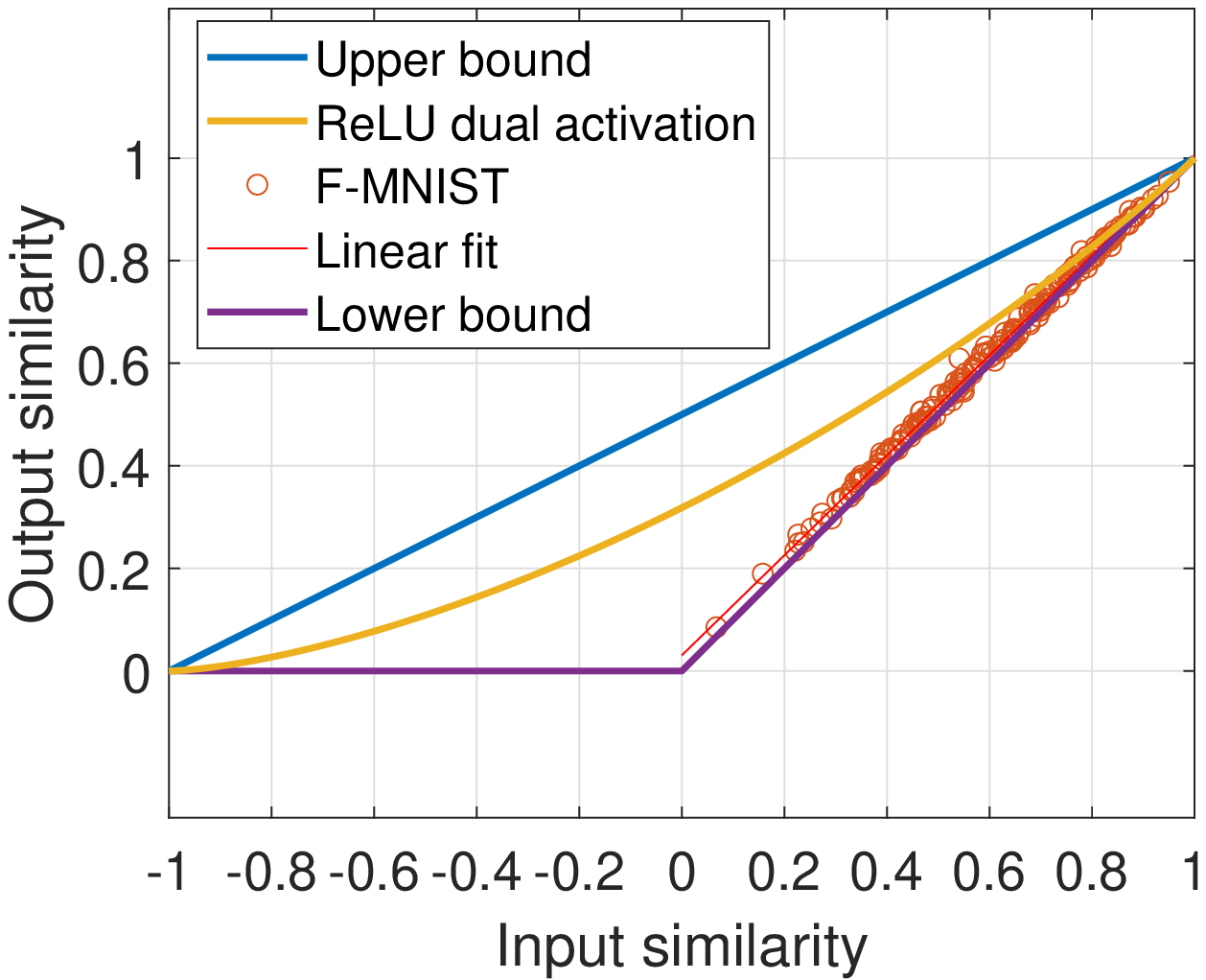} 
        \caption{F-MNIST, filter size $3\times 3$}
    \label{fig:fmnist}
    \end{subfigure}\hfill
     \begin{subfigure}{0.45\textwidth}
        \includegraphics[width=\linewidth]{  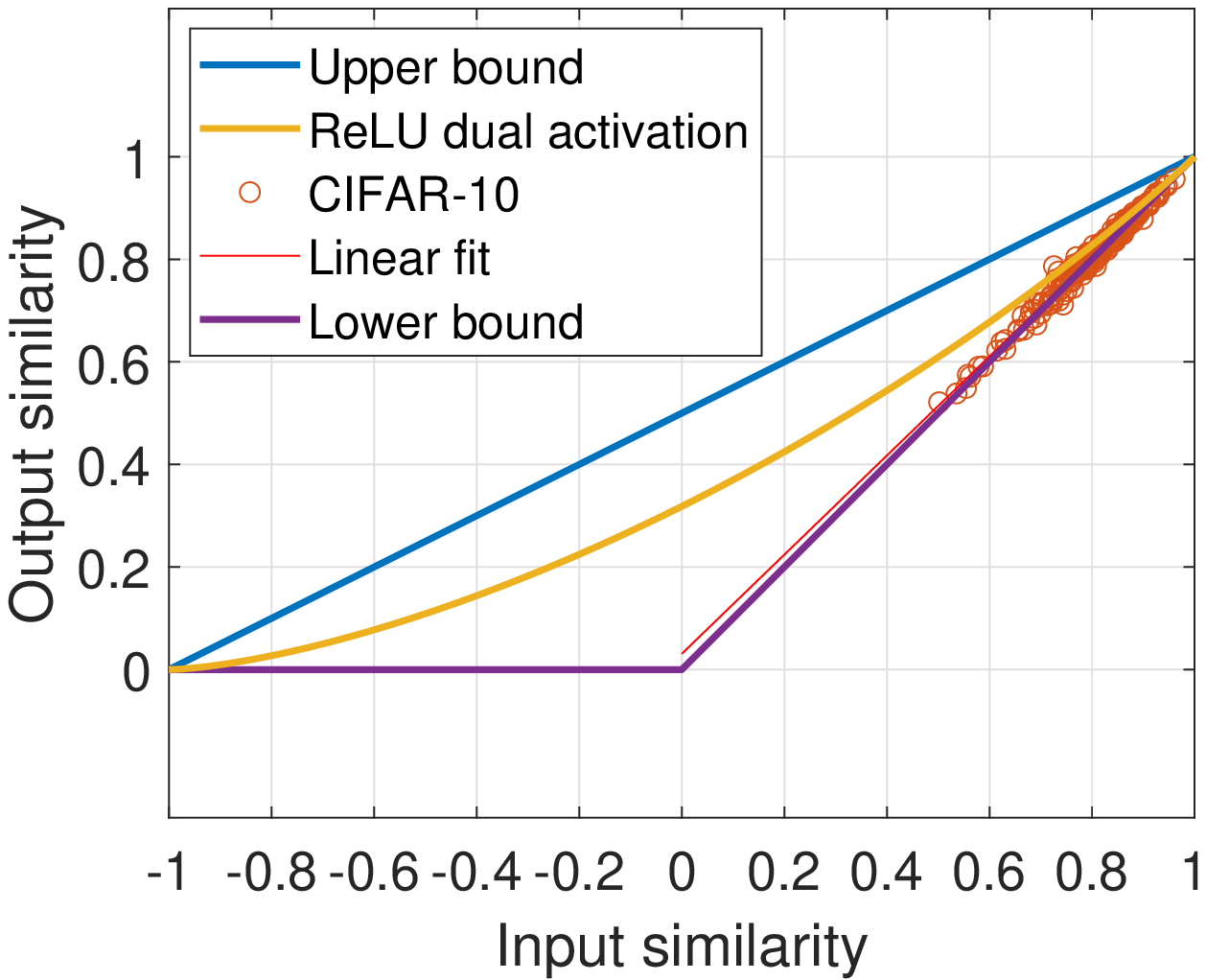}
        \caption{CIFAR-10, filter size $3\times 3\times 3$}
    \label{fig:cifar-10}
    \end{subfigure}\hfill
     \begin{subfigure}{0.45\textwidth}
        \includegraphics[width=\linewidth]{  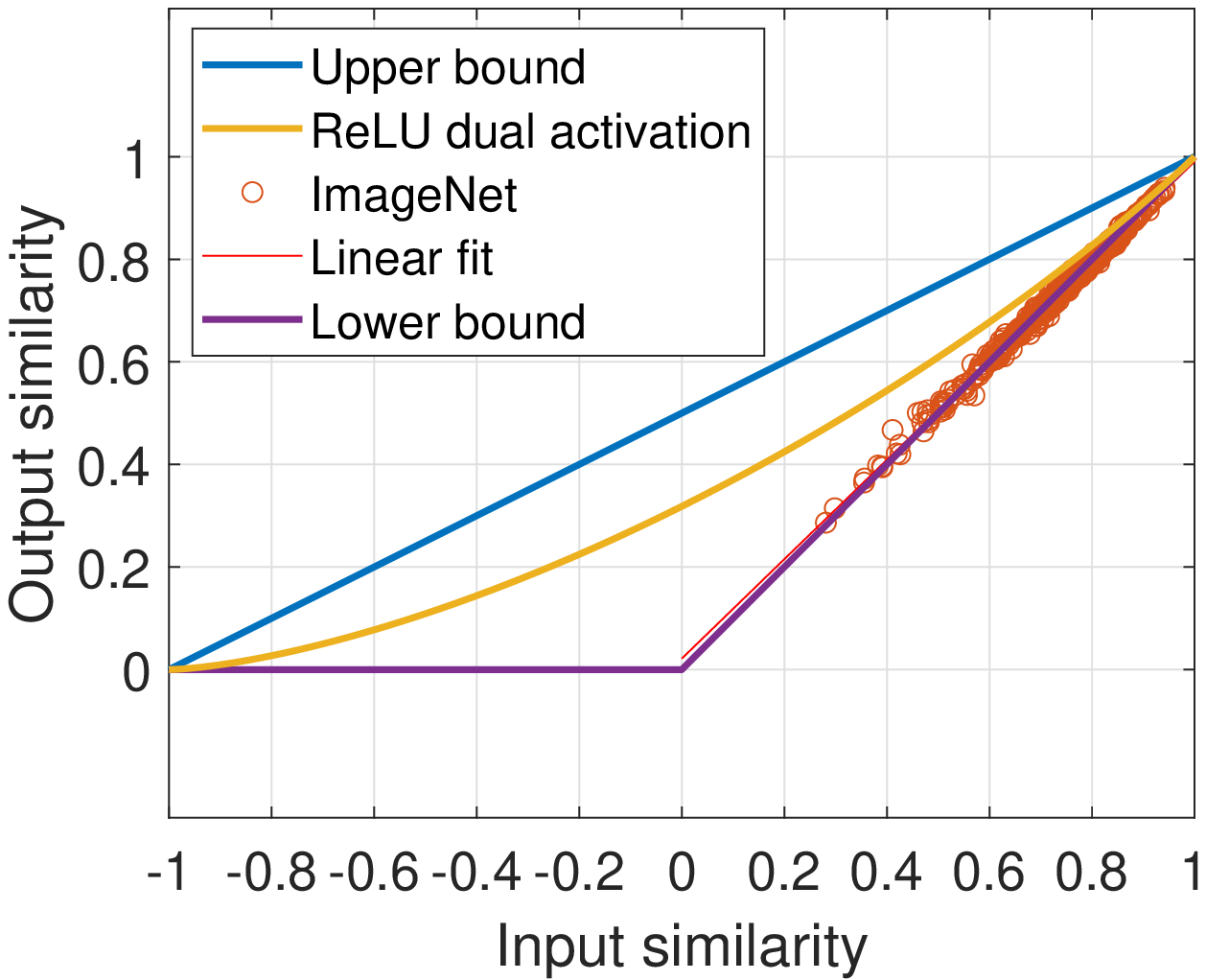}
        \caption{ImageNet, filter size $11\times 11\times 3$}
    \label{fig:imagenet}
    \end{subfigure}
      \caption{Input and output cosine similarities of a single  randomly initialized convolutional layer with 100 filters. Each red circle in the figures represents a random pair of images chosen from the corresponding dataset. 200 pairs were sampled in each figure. }
\label{fig:transform}
 \vspace{-1.5\baselineskip}
\end{figure}

Following \citet{daniely}, we refer to the resulting function in \eqref{eq:rho_in-rho_out} of $\brho_\OUT$ in $\rho_\IN$ as the \textit{dual activation} of the ReLU; this function is represented 
in \figref{fig:transform} by the yellow curve.

One may easily verify that the dual activation \eqref{eq:rho_in-rho_out} 
of the ReLU satisfies $\rho_{\OUT} > \rho_{\IN}$ for $\rho_\IN \neq 1$, meaning that it is a \textit{contraction}.
Consequently, for deep FNNs, which comprise multiple layers, random initialization 
results in a \textit{collapse of all inputs} with the same norm (a sphere) 
\textit{to a single point} at the output of the FNN (equivalently, the entire dataset collapses to a single straight line).


{
Intuitively, this collapse is an unfavorable starting point for optimization. To see why, consider the gradient 
$\nabla_{w^{(i)}_j} $
of the weight $w^{(i)}_j$ of some neuron $j$ in a deep layer $i$ 
in a randomly initialized ReLU FNN. 
By the chain rule (backpropagation), 
this gradient is proportional to the output of the previous layer $a^{(i-1)}$ for the corresponding input,
i.e., it holds that
$\nabla_{w^{(i)}_j} \propto a^{(i-1)}$. 
If the collapse is present already at layer $i$, this output is
essentially proportional to a fixed vector $a^*$. But this implies
that, in the gradient update, the weights of the deep layer will move roughly along a straight line which would impede, in turn, the process of achieving linear separability.
}

Indeed, it is considered that for a FNN to train well, its input--output Jacobian needs to exhibit \emph{dynamical isometry} upon initialization \citep{Saxe14exactsolutions}. Namely, the singular values of the Jacobian $\partial x_\OUT /\partial x_\IN$ must be concentrated around $1$, 
where $x_\OUT$ and $x_\IN$ denote the input and output of the FNN, respectively. 
If the dataset collapses to a line,
$x_\OUT$ is essentially invariant to $x_\IN$ (up to a change in its norm), 
suggesting that the singular values of $\partial x_\OUT /\partial x_\IN$ are close to zero. 
Therefore, randomly initialized FNNs exhibit the \textit{opposite} behavior from dynamical isometry and hence do not train well.



\subsection{Our Contribution}
\label{ss:intro:contributions}

Our main interest lies in the following question.

\textit{Does the contraction observed in randomly initialized ReLU FNNs 
carry over to convolutional neural networks (CNNs)?}

As we will show, qualitatively, the answer is yes.
However, quantitatively, the answer is more subtle as is illustrated in \figref{fig:transform}. In this figure, the similarity between pairs of inputs sampled at random from standard natural image datasets---Fashion MNIST (F-MNIST), CIFAR-10, and ImageNet---are displayed against the corresponding output of a randomly initialized CNN layer. For these datasets, clearly $\varrho_\OUT \approx \rho_\IN$ meaning that the relation of ReLU FNNs \eqref{eq:rho_in-rho_out}---represented in \figref{fig:transform} by the yellow curve---breaks down.

That said, for inputs consisting of i.i.d.\ zero-mean Gaussians (and filters comprising i.i.d.\ zero-mean Gaussian weights as before) with a Pearson correlation coefficient $\rho$ between corresponding entries (and independent otherwise), the relation in \eqref{eq:rho_in-rho_out} between $\brho_\OUT$ and $\rho_\IN$ of ReLU FNNs does hold for ReLU CNNs as well, as illustrated in \figref{fig:gaussian}.


This dataset-dependent behavior, observed in \figref{fig:transform}, 
suggests that,
in contrast to randomly-initialized FNNs which behave according to \eqref{eq:rho_in-rho_out},
randomly-initialized CNNs exhibit a richer behavior: $\brho_\OUT$ does not depend only on $\rho_\IN$ but on the inputs $x_\IN$ and $y_\IN$ themselves. 
Therefore, in this work, we characterize the behavior of $\brho_\OUT$ after applying one layer
in randomly initialized CNNs.

We start by considering randomly initialized CNNs with general activation functions.
We show in \thmref{thm:JL-general} that 
the expected (over the filters) inner product $\E{\inner{X_\OUT}{Y_\OUT}}$ and the mean similarity $\brho_\OUT$ depend on $x_\IN$ and $y_\IN$ (and not just $\inner{x_\IN}{y_\IN}$) by extending the dual-activation notion of \citet{daniely}.
In \thmref{thm:JL-general:CoM}, we further prove that, by taking multiple independent filters, $\inner{X_\OUT}{Y_\OUT}$ and $\varrho_\OUT$ of \eqref{eq:rho_out:def} concentrate around $\E{\inner{X_\OUT}{Y_\OUT}}$ and $\brho_\OUT$, respectively.

We then specialize these results to linear CNNs (with identity activation) and derive a convolution-based variant of the Johnson--Lindenstrauss lemma that shows that
$\inner{X_\OUT}{Y_\OUT} \approx \inner{x_\IN}{y_\IN}$ and $\varrho_\OUT \approx \rho_\IN$ for linear CNNs, both in expectation ($\brho_\OUT = \rho_\IN$ for the latter) and with high probability.

For randomly initialized ReLU CNNs, we derive
the following tight upper and lower bounds for $\brho_{\OUT}$ in terms
of $\rho_{\IN}$ in \thmref{thm:con_map}:
\begin{align} 
\label{eq:rho_out>=rho_in}
    \max \{\rho_\IN, 0\} \leq \brho_\OUT \leq \frac{1 + \rho_\IN}{2} .
\end{align} 
These bounds imply, in turn, that for $\rho_\IN \neq 1$ each ReLU CNN layer is contracting.
In \thmref{thm:gaussian} we prove that $\brho_\OUT$ for random Gaussian data satisfies the relation~\eqref{eq:rho_in-rho_out}, in accordance
with \figref{fig:gaussian}.



To explain the (almost) isometric behavior of CNNs for natural images (\figref{fig:transform}),
we note that many natural images consist of large, relative to the filter size, approximately monochromatic patches. This observation leads 
to 
a simple model of black and white (binary) images with ``large patches''.
To describe this model mathematically, 
we define a notion of a shared boundary between two images
in \defref{def:boundary}, 
and model large patches by bodies whose area is large compared to the shared boundary.
We prove that $\brho_\OUT \approx \rho_\IN$ for this model, meaning
that the lower bound in \eqref{eq:rho_out>=rho_in} is in fact tight.

\subsection{Related Work}


In this paper, we study how various inputs are embedded by randomly initialized convolutional neural networks. Neural tangent kernel (NTK)  \citep{jacot2018neural} is a related line of work. This setting studies the infinite width limit (among other assumptions) in which one can consider neural network training as regression over a fixed kernel;  this kernel is the NTK. 
There are two factors that affect the calculation of the NTK: The embedding of the input space at initialization and  the gradients at initialization. In this paper we study the first one.

\citet{arora2019exact}, and \citet{bietti2019inductive} give expressions for the NTK and the convolutional NTK (CNTK). \Thmref{thm:JL-general} may be implicitly deduced from those expressions.  \citet{arora2019exact}  provide concentration bounds for the NTK of fully connected networks with finite width. \citet{bietti2019inductive} derive smoothness properties for NTKs, 
e.g.,
upper bounds on the deformation induced  by the NTK in terms of the initial Euclidean distance between the inputs. A related approach to NTK is taken in \citep{bietti2021approximation} where convolutional kernel networks (CKN) are used.


Standard initialization techniques use the Glorot initializtion \citep{pmlr-v9-glorot10a} and He initialization \citep{kaiming}. Both were introduced to prevent the gradients from exploding/vanishing. 
On a similar note, \citet{hanin2018start} discuss how to prevent exploding/vanishing mean activation length---which corresponds to the gradients to some degree---in FNN with and without skip connections. For a comprehensive review on other techniques see \citep{init_review}.

\citet{schoenholz2016deep,poole2016exponential,NIPS2017_81c650ca} study the initialization of FNNs using mean-field theory, in a setting where the width of the network is infinite. They demonstrate that for some activation functions there exists an initialization variance such that the network does not suffer from vanishing/exploding gradients. In this case, the network is said to be initialized at the edge of chaos. 

As mentioned earlier, \citet{Saxe14exactsolutions} introduced a stronger requirement than that of moderate gradients at initialization, in the form of \emph{dynamical isometry}. for linear FNNs, they showed that orthogonal initialization achieves dynamical isometry whereas Gaussian i.i.d.\ initialization does not. 

For non-linear FNNs, \citet{pennington2017resurrecting} show that the hyperbolic tangent (tanh) activation can achieve dynamical isometry while ReLU FNNs with Gaussian i.i.d.\ initialization or  orthogonal initialization cannot. 
In contrast, \citet{burkholz2019initialization} show that ReLU FNNs achieve dynamical isometry by moving away from i.i.d.\ initialization. And lastly, \citet{tarnowski2019dynamical} show that residual networks achieve dynamical isometry  over a broad spectrum of activations (including ReLU) and initializations.

\citet{xiao2018dynamical} trained tanh (and not ReLU) CNNs with 10000 layers that achieve dynamical isometry using the delta-orthogonal initialization. Similarly, \citet{zhang2019fixup} trained residual CNNs with 10000 layers  without batch normalization using fixup-initialization that prevents exploding/vanishing gradients.

\subsection{Notation}

\begin{longtable}{cp{.88\textwidth}}
        $\sR, \sN, \sZ$                         & The sets of real, natural, and integer numbers, respectively.
\\[2pt] $\lor, \land$                           & The logical `or' and `and' operators, respectively.
\\[2pt] $\sZ_n$, $[n]$                          & The sets $\{0, 1, \ldots, n-1\}$ and $\{0, 1, \ldots, n-1\}$.
\\[2pt] $A_{ij}$, $T_{ijk}$                     & The $(i,j)$ entry of the matrix $A$ and the $(i,j,k)$ entry of the tensor $T$, respectively.
\\[2pt] $[T]^r_{ij}$                            & The sub-tensor of $T$ formed from rows $i$ to $i+r-1$ and columns $j$ to $j+r-1$ (and full third ``layer'' dimension). 
\\[2pt] $I_k$                                   & The identity matrix of dimensions $k\times k$. 
    %
\\[2pt] $\mathrm{vec}(\cdot)$                   & The vectorization operation of a matrix or a tensor in some systemic order. 
\\[2pt] $\inner{x}{y}$                          & The standard inner product between vectors $x$ and $y$. For matrix or tensor $x$ and $y$ (of the same dimensions) the inner product is defined as $\inner{x}{y} = \inner{\mathrm{vec}(x)}{\mathrm{vec}(y)}$. 
\\[2pt] $\norm{x}$                              & The standard Euclidean norm, induced by the standard inner product: $\norm{x} = \sqrt{\inner{x}{x}}$. 
\\[2pt] $\rho$                                  & The (cosine) similarity. The similarity between $x$ and $y$ is defined as $\rho = \frac{\inner{x}{y}}{\norm{x} \norm{y}}$. 
\\[4pt] $\brho$                                 & The mean similarity. $\brho$ between random $X$ and $Y$ is defined as $\brho = \frac{\E{\inner{X}{Y}}}{\sqrt{\E{\norm{X}^2} \E{\norm{Y}^2}}}$.
\\ $\sigma$, $\hat{\sigma}$                          & An activation function and its dual, respectively, as defined in \defref{def:dual-activation}.
\\[2pt] $\ReLU(x)$                              & The ReLU applied to $x$: $\ReLU(x) := \max\{0, x\}$; for a tensor $x$ the $\ReLU$ is applied entrywise.
\\ $\Rone(\rho)$                           & The dual activation of the ReLU at $\rho \in [-1,1]$: $\Rone(\rho) := \frac{\sqrt{1 - \rho^2} + \left( \pi - \cos^{-1}(\rho) \right) \rho}{\pi}$.
\\[2pt] $\gN(\mu, C)$                       & A Gaussian vector with mean vector $\mu$ and covariance matrix $C$.
\end{longtable}

\section{Setting and Definitions}
\label{s:model}


The output $z \in \RR^{n \times n}$ of a single filter of a CNN is given by 
\begin{align}
    z = \sigma(F * x) 
\end{align}
where $x \in \RR^{n \times n \times d}$ is the input;
$F \in \RR^{r \times r \times d}$ is the filter (or kernel) of this layer;
and $\sigma: \RR \to \RR$ is the activation function of this layer and is understood to apply entrywise when applied to tensors.
The dimensions satisfy $n,d,r \in \sN, r\le n$.
The first two dimensions of $x$ and $z$ are the image dimensions (width and height), while the third dimension is the number of channels, and is most commonly equal to either 1 for grayscale images and 3 for RGB color images; $r$ is the dimension of the filter;
`$*$' denotes the cyclic convolution operation:
\begin{align}
\label{eq:def:convolution}
    (F * x)_{uv} = \sum_{i \in \sZ_r,j \in \sZ_r, k \in \sZ_d} 
    F_{ijk} \cdot x_{u-i,v-j,k}
\end{align}
where the subtraction operation in the arguments is the subtraction operation modulo $n$. 
In this work, 
we will concentrate on the ReLU activation function which is given by $\sigma(x) = \max\{0,x\}$. 
Typically, a single layer of a CNN consists of multiple filters
$F_1,\ldots, F_N$.




The following definition is an adaptation of the dual activation of \citet{daniely} to CNNs.
\begin{definition}[Dual activation]
\label{def:dual-activation}
    The dual activation function $\dualactthree :\RR^n \times \RR^n \times \RR \to \RR $ of an activation function $\sigma$ is defined as 
    \begin{align} 
    \label{eq:def:dual-activation-three}
        \dualactthree(x,y,\nu)=\E{\sigma (\nu X) \cdot 
        \sigma (\nu Y)}
    \end{align} 
    for input vectors $x,y \in \RR^n$, and a parameter $\nu>0$,
    where 
    \begin{align}
    \label{eq:def:dual-activation:Gauss-vec}
        \begin{pmatrix} X \\ Y \end{pmatrix} \sim \gN\left( \begin{pmatrix} 0 \\ 0 \end{pmatrix}, 
         \begin{pmatrix}
            \lv x \rv ^2 &  \inner{x}{y} 
         \\[3pt] 
            \inner{x}{y} &  \lv y \rv ^2 
         \end{pmatrix}
         \right).
     \end{align}
     For tensors $x,y \in \RR^{n \times n \times d}$, 
     the dual activation $\dualactthree$ is defined as in \eqref{eq:def:dual-activation:Gauss-vec}
     with the inner products and norms replaced by the standard inner product between tensors:
     \begin{align} 
        \inner{x}{y} &= \sum_{i \in \sZ_n, j \in \sZ_n, k \in \sZ_d} x_{ijk}\cdot y_{ijk}, 
      & \norm{x} &= \sqrt{\inner{x}{x}} .
     \end{align}
     In the special case of an activation that is homogeneous,
     i.e., satisfying $\sigma(cx)=c\sigma(x)$ for $c\ge 0$,
     we use an additional notation 
     for $-1\leq \rho \leq 1$, with some abuse of notation:
      \begin{align} 
    \label{eq:def:dual-activation}
        \nu^2_\sigma := \E{\sigma(X)^2}\;,
        ~~~~~~
        \dualactone(\rho):=\frac{
        \E{\sigma (X) \cdot \sigma (Y)}}
        {\nu^2_\sigma}\;,
        ~~~~~~ 
        \begin{pmatrix} X \\ Y \end{pmatrix} \sim N\left( \begin{pmatrix} 0 \\ 0 \end{pmatrix}, 
         \begin{pmatrix}
            1 & \rho
         \\[3pt] 
           \rho  &  1 
         \end{pmatrix}
         \right).
    \end{align} 
    
This is the dual activation 
function 
of \citet{daniely}.
Note that, by this definition, $\dualactone(1)=1$.
\end{definition}

We will be mostly interested in the ReLU activation, denoted by $\ReLU$, for which the following holds (see, e.g., \citep[Table~1 and Section~C of supplement]{daniely}) 
\begin{align}
\nu^2_{\ReLU}:=1/2\;,
~~~~~~
\R(\rho):=\frac{\E{\ReLU (X) \cdot \ReLU (Y)}}{\nu^2_{\ReLU}} = \frac{\sqrt{1 - \rho^2} + \left( \pi - \cos^{-1}(\rho) \right) \rho}{\pi}\;.
\end{align}

\section{Main Results}
\label{s:main}

In this section, we present the main results of this work. \textbf{The proofs of all the results in this section may be found in the appendix.}

\begin{theorem}
\label{thm:JL-general}
    Let $x,y\in\RR^{n\times n\times d}$ be inputs to a convolution filter $F\in \RR^{r\times r\times d}$ with $r \leq n$ and activation function $\sigma$ such that the entries of $F$ are i.i.d.\ Gaussian with variance $\nu^2$. 
    Then, 
    \begin{align}\label{eq:10}
        \E{ \inner{\sigma (F * x)}{\sigma( F * y )} } = 
        \sum_{i,j \in \sZ_n} \dualactthree \left(  [x]^r_{ij},[y]^r_{ij},\nu \right).
    \end{align} 
    In particular, for a homogeneous activation function 
    we have 
    \begin{align}
    \E{ \inner{\sigma (F * x)}{\sigma( F * y )} } 
    &=\nu^2\cdot \nu_\sigma^2 \sum_{i,j \in \sZ_n} \lt{ [x]^r_{ij}}  \lt{[y]^r_{ij}} \dualactone \left(  \rho_{ij}  \right),\label{eq:07}\\
    \E{\sigma(F*x)^2}&=\nu^2\cdot\nu_{\sigma}^2
    \cdot r^2 \|x\|^2\;,
    \label{eq:08}
    \end{align}
where $\rho_{ij}=\frac{\inner{[x]^r_{ij}}{[y]^r_{ij}} }  {\lv [x]^r_{ij} \rv \cdot \lv [y]^r_{ij} \rv}$
and $\nu_{\sigma}^2$ has been defined in~\eqref{eq:def:dual-activation}.
\end{theorem}

In particular, due to~\eqref{eq:08}, for homogeneous
activations we can choose $\nu=1/(\nu_{\sigma}r)$ and
get $\E{\sigma(F*x)^2}=\|x\|^2$. That is, the norms
are preserved in expectation for every input $x$.

    Note that
    \thmref{thm:JL-general} may be deduced as a
    special case from existing more general formulas;
    see, e.g., \citet{arora2019exact, bietti2019inductive}. Nevertheless, it is an
    important starting point for us.


While \thmref{thm:JL-general} holds in the mean sense, it does not hold for a specific realization of a single filter $F$, in general. 
The next theorem, whose proofs relies on the notions of sub-Gaussian and sub-exponential RVs and their concentration of measure, 
states that \thmref{thm:JL-general} holds approximately for a large enough number of applied filters. 

We apply basic,
general concentration bounds on the activation. 
There exist sharper bounds in special cases, e.g., in
the context of cosine similarity and ReLU, see \citep[app.~E]{buchanan2020deep}.
\begin{theorem}
\label{thm:JL-general:CoM} 
    Let $x,y\in\RR^{n\times n\times d}$ be inputs to $N$ filters  $F_1,\ldots,F_N\in \RR^{r\times r\times d}$ with $r \leq n$ such~that 
    \begin{align}
    \label{eq:thm:JL-general:CoM:norm-bound}
        \max_{i,j \in \sZ_n} \lt{ [x]^r_{ij}} &\leq R, 
      & \max_{i,j \in \sZ_n} \lt{ [y]^r_{ij}} &\leq R,
   \end{align}
    all the entries of all the filters are i.i.d.\ Gaussian with zero mean and variance $\nu^2$, and a Lipschitz continuous activation function $\sigma$ with a Lipschitz constant $L$ 
   and satisfying $\sigma(0)=0$. 
    Then, for $\delta > 0$, however small, 
    \begin{align}
    \label{eq:thm:CoM}
        \sP \left( \left| \frac{1}{N} \sum_{\ell=1}^N \inner{\sigma (F_\ell * x)}{\sigma( F_\ell * y )}   - \sum_{i,j=1}^n \dualactthree\left(  [x]^r_{ij},[y]^r_{ij},\nu \right)\right| \geq  \epsilon \right) \leq \delta
    \end{align} 
    for $N >  \max \left(K, K^2 \right) \log \frac{2 n^2}{\delta}$, 
    where 
    $K = D\nu^2 L^2 R^2 n^2 / \eps$ and $D > 0$ is an absolute constant.
\end{theorem}
\begin{remark}
    Clearly, $\|[x]_{ij}^r\|$ and $\|[y]_{ij}^r\|$ in \eqref{eq:thm:JL-general:CoM:norm-bound} may be bounded by $\norm{x}$ and $\norm{y}$, respectively. 
    However, since $r$ is typically much smaller than $n$, 
    we prefer to state the result in terms of the norms of $[x]_{ij}^r$ and $[y]_{ij}^r$.
\end{remark}

\Thmref{thm:JL-general:CoM} states a concentration
result in terms of the inner product.
We now give a parallel result for homogeneous activations
in terms of the cosine similarity $\varrho_\OUT$.
To that end,
consider a homogeneous activation $\sigma$ and $N$ random filters
$F_1,\ldots,F_N$;
for these filters, the quantities of \eqref{eq:rho_out:def} 
for two inputs $x$ and $y$ are equal to 
\begin{align}\label{eq:11}
    \varrho_\OUT &=
    \frac{\sum_{\ell=1}^N\inner{\sigma \left( F_\ell*x \right)}{\sigma \left( F_\ell*y \right)}}
    {\sqrt{\sum_{\ell=1}^N \left\|\sigma \left( F_\ell*x \right)\right\|^2}
    \sqrt{\sum_{\ell=1}^N \left\|\sigma \left( F_\ell*y \right)\right\|^2}},
  & \brho_\OUT & =
    \frac{
    \sum_{i,j\in\mathbb{Z}_n}\|[x]_{ij}^r\|\cdot
    \|[y]_{ij}^r\|\dualactone \left( \rho_{ij} \right)
    }{r^2\|x\|\cdot\|y\|},
\end{align}
where $\rho_{ij}$ is defined in \thmref{thm:JL-general}.
Applying \thmref{thm:JL-general:CoM} to 
these quantities, yields the following. 
\begin{corollary}\label{cor:similarity}
Let $x,y$ be inputs to $N$ filters $F_1,\ldots,F_N$ such that 
\begin{align}
    \max_{i,j\in\mathbb{Z}_n} \frac{\|[x]_{ij}^r\|}{\|x\|}\le R\;,
    ~~~~~~~~
    \max_{i,j\in\mathbb{Z}_n} \frac{\|[y]_{ij}^r\|}{\|y\|}\le R\;,
\end{align}
all the
entries of the filters are i.i.d.\ Gaussian with zero mean and variance
$\nu^2$, and a Lipschitz homogeneous activation $\sigma$ with Lipschitz
constant $L$. Then, for $0<\epsilon\le1/10$ and $\delta>0$,
\begin{align}
    \sP\big(\left|\varrho_\OUT(x,y)-\brho_\OUT(x,y)\right|
    \ge\eps\big)\le\delta
\end{align}
for $N>\max(K,K^2)\log\frac{6n^2}{\delta}$, where
$K=\frac{DL^2R^2n^2}{\epsilon\nu_{\sigma}^2r^2}$ and $D>0$ is an absolute constant.
\end{corollary}

\begin{remark}
    To make sense out of the constants in
    \corref{cor:similarity}, we point out that
    applying it to inputs 
    $x,y\in\{\pm 1\}^{n\times n\times d}$
    (so that $\|x\|^2=\|y\|^2=dn^2$ and
    $\|[x]_{ij}^r\|^2=\|[y]_{ij}^r\|^2=dr^2$)
    results in $K=\frac{DL^2}{\epsilon\nu_{\sigma}^2}$,
    which does not depend on $n$ or $\nu^2$.
    There remains a $\log n^2$ factor that increases with
    $n$.
\end{remark}

In the remainder of this section, we derive an analogous result to the
Johnson--Lindenstrauss lemma \citep{Johnson-Lindenstrauss-Lemma:original1984,Johnson-Lindenstrauss-Lemma:Dasgupta-Gupta:probabilistic-proof:1999} 
in \secref{ss:main:linearCNN}
for random \textit{linear} CNNs, namely, CNNs with identity activation functions. We then move to treating random CNNs with ReLU activation functions: We derive upper and lower bounds on the inner product between the outputs in \secref{ss:main:bounds}, and prove that they are achievable for Gaussian inputs in \secref{ss:main:GaussianIn}, and for large convex-body inputs in \secref{ss:main:convex-body}, respectively.

\vspace{-.2\baselineskip}
\subsection{Linear CNN}
\label{ss:main:linearCNN}
\vspace{-.2\baselineskip}

\Thmref{thm:JL-general} yields a variant of the Johnson--Lindenstrauss lemma 
where the random projection is replaced with random filtering,
that is, by applying a (properly normalized) random filter, the inner product is preserved (equivalently, the angle).
 
\begin{lemma}
\label{lem:JL}
    Consider the setting of \thmref{thm:JL-general} with $\nu = 1/r$.
    Then, 
        $\E{\inner{F * x}{F * y}} =  \inner{x}{y}$.
\end{lemma}


\begin{lemma}
\label{lem:JL:CoM}
    Consider the setting of \thmref{thm:JL-general:CoM} with $\nu = 1/r$ (and $\sigma(x) \equiv x$).
    Then, 
    \begin{align}
    \\[-1.2\baselineskip]
        \sP \left( \left| \frac{1}{N} \sum_{\ell=1}^N \inner{F_\ell * x}{F_\ell * y}   - \inner{x}{y} \right| \geq  \epsilon \right) \leq \delta
    \end{align} 
    for $N > \max \left( K, K^2 \right) \log \frac{2n^2}{\delta}$, 
    where $K = \frac{  D R^2 n^2 }{r^2 \epsilon}$ and $D > 0$ is an absolute constant.
\end{lemma}

Related result appears in~\citet{krahmer2014suprema} (theorem 1.1) which shows norm preservation of sparse vectors for convolutional operators with a filter dimension that equals the input dimension.

\vspace{-.2\baselineskip}
\subsection{CNN with ReLU Activation}
\label{ss:main:bounds}
\vspace{-.2\baselineskip}



The following theorem implies that every layer of a neural network is contracting in expectation (over $F$). 
That is, the expected inner product between any two data points will get larger with each random CNN layer.
We also develop an upper bound on the new correlation.

\begin{theorem}
\label{thm:con_map}
    Consider the setting of \thmref{thm:JL-general} with a ReLU activation function and variance $\nu =2/r^2$, 
    for inputs $x$ and $y$ with similarity $\rho$.
    Then, 
    \begin{align}
    \\[-1.2\baselineskip]
    \label{eq:thm:con_map:bounds}
        \max  \{ \inner{x}{y} ,0  \}  
        \leq \E{ \inner{\ReLU (F * x )}{\ReLU( F * y ) }} \leq \lt{x} \lt{y} \frac{1 + \rho}{2} ,
    \\[-1.1\baselineskip]
    \end{align} 
\end{theorem}

\vspace{-.2\baselineskip}
\begin{remark}\label{rem:tightness}
All the bounds in \thmref{thm:con_map} are tight.
For simplicity, we present examples where the unit vectors
$x, y$ are flat, i.e., they are not tensors and $r=1$. Similar examples can be drawn for tensor inputs and $r>1$. 
The upper bound is realized for $x,y\in\mathbb{R}^2$ of the form
$x=\left(\sqrt{\frac{1+\rho}{2}},\sqrt{\frac{1-\rho}{2}}\right)$,
$y=\left(\sqrt{\frac{1+\rho}{2}},-\sqrt{\frac{1-\rho}{2}}\right)$.
Similarly, $x,y\in\mathbb{R}^3$ with
$x=(\sqrt{1-\rho},\sqrt{\rho},0)$, $y=(0,-\sqrt{\rho},\sqrt{1-\rho})$
satisfy $\inner{x}{y}=-\rho$ and
$\E{\sigma(F*x)\sigma(F*y)}=0$ for $r=1$ and $0\le\rho\le 1$.
Finally, we can take $x=(\sqrt{1-\rho},\sqrt{\rho},0)$
and $y=(0,\sqrt{\rho},\sqrt{1-\rho})$ to obtain
$\inner{x}{y}=\E{\sigma(F*x)\sigma(F*y)}=\rho$.

These examples are illustrated in \figref{fig:tightness}.
\end{remark}

\begin{remark}
    Although the lower bound is tight, in most typical scenarios it will be strict, so, in expectation, the convolutional layers will \emph{contract} the dataset after application of enough layers.   To 
    see why, observe \eqref{eq:proof:general:CoM:max}; equality is 
    achieved iff $\rho_{ij}=1 \vee [x]^r_{ij} = 0 \vee [y]^r_{ij}=0 $. For real data sets this will never hold over all $(i,j)$. This can be observed in \trifigref{fig:fmnist}{fig:cifar-10}{fig:imagenet} where the linear fit has a slope of 0.97.
\end{remark}


\subsection{Gaussian Inputs}
\label{ss:main:GaussianIn}

The following result 
shows that the behavior reported by \citep{daniely} [and illustrated in \figref{fig:gaussian} by the orange curve] for FNNs holds for CNNs with Gaussian correlated inputs; this is illustrated in \figref{fig:gaussian}.
For simplicity, we state the result for $d=1$ ($n \times n$ inputs instead of $n \times n \times d$); the extension to $d > 1$ is straightforward.

\begin{theorem}
\label{thm:gaussian}
    Let $X,Y\in \RR^{n \times n}$ be zero-mean jointly Gaussian with the following auto- and cross-correlation functions:
    \begin{align}
    \\[-1.3\baselineskip]
    \label{eq:thm:gaussian-inputs:correlations}
        \E{X_{ij} X_{k\ell}} &= \E{Y_{ij} Y_{k\ell}} = \frac{1}{n^2} \delta_{ik} \delta_{j\ell} ,
      & \E{X_{ij} Y_{k\ell}} &= \frac{\rho}{n^2} \delta_{ik} \delta_{j\ell} ,
    \end{align}
    i.e., $X$ and $Y$ have pairwise i.i.d.\ entries which are correlated with a Pearson correlation coefficient $\rho$.
    Assume further that the filter $F \in \RR^{r \times r}$ comprises zero-mean i.i.d.\ Gaussian entries with variance $2/r^2$ and is independent of $(X,Y)$. 
    Then, for a ReLU activation function $\sigma$, 
    \begin{align} 
        \E{ \inner{\sigma(F*X)}{\sigma(F * Y)} }= \frac{\sqrt{1-\rho^2} +(\pi - \cos^{-1}(\rho))\rho }{\pi} .
    \end{align}
\end{theorem}

\subsection{Simple Black and White Image Model}
\label{ss:main:convex-body}

This section provides insight why we get roughly an isometry when embedding natural images via a random convolution followed by ReLU. 

As a conceptual model, we will work with binary pictures $\{0,1\}^{ n \times n}$ or equivalently as subsets of $\mathbb{Z}_n \times \mathbb{Z}_n$.  Also, when considering real-life pictures, an apparent characteristic is that they consist of relatively large monochromatic patches, so for the next theorem, one should keep in mind pictures that consist of large convex bodies.  

To state our results, we define a notion
of a shared $r$-boundary.

\begin{definition}[Boundary]\label{def:boundary} 
Let $A,B \subset \mathbb{Z}_n \times \mathbb{Z}_n $ and let 
$r \in [n]$. 
Then, the $r$-boundary of the intersection between $A$ and $B$, denoted by $\partial_r (A , B)$, is defined as the set of all pixels $(i,j) \in 
\mathbb{Z}_n\times\mathbb{Z}_n$ such that 
\begin{subequations}
\begin{align}
    \exists a_1,b_1,c_1,d_1\in\{-r,\ldots,r\}&:
        (i+a_1,j+b_1)\in A\land
        (i+c_1,j+d_1)\in B\label{eq:05}\\
    \exists a_2,b_2,c_2,d_2\in\{-r,\ldots,r\}&:
        (i+a_2,j+b_2)\notin A\lor
        (i+c_2,j+d_2)\notin B\label{eq:06}
\end{align}
\end{subequations}
where the addition in~\eqref{eq:05} and~\eqref{eq:06} is over $\mathbb{Z}_n$.
\end{definition}

In words, a pixel $(i,j)$ belongs to the $r$-boundary
$\partial_r(A,B)$ if, considering the 
square with edge size $2r+1$ centered at the pixel: (1) the square intersects both $A$ and $B$, (2)
the square is not contained in $A$ or is not contained in $B$.

\begin{example}\label{ex}
    Let $A$ and $B$ be axis-aligned rectangles of sizes $a \times b$ and $c \times d$, respectively, and an intersection of size $e \times f$ where $e,f\geq 2r$.  Then, $ |\partial_r (A , B)|  = 4r(e+f)$. This is illustrated in \figref{fig:boundary}.
\end{example}
\vspace{-.3\baselineskip}



The next theorem bounds the inner product between $A$ and $B$ after applying a ReLU convolution layer in terms of the $r$-boundary
of the intersection of $A$ and $B$.

\begin{theorem}
\label{thm:boundary}
    Let $A,B \subset [n] \times [n] $  and a convolution filter  $F\in \RR^{ (2r+1)  \times  (2r+1) } $ such that the entries in $F$ are i.i.d. Gaussians  with variance $2/(2r+1)^2$. Then, 
    \begin{subequations}
    \begin{align}
        \inner{A}{B} - |\partial_r (A , B)|  
        &\leq \E{ \inner{\ReLU(F*A)}{\ReLU(F * B)} } \leq \inner{A}{B} + |\partial_r (A , B)|\;,
    \label{eq:04}
    \\ \|A-B\|^2-2|\partial_r(A,B)|
        &\le \E{ \|\ReLU(F*A)-\ReLU(F*B)\|^2 }
        \le \|A-B\|^2+2|\partial_r(A,B)|\;.
    \quad\ 
    \label{eq:09}
    \end{align}
    \end{subequations}
\end{theorem}

\Thmref{thm:boundary} and \corref{cor:similarity} imply 

\begin{align} 
\\[-2.2\baselineskip]
    \left| \varrho_\OUT - \rho_\IN \right| \leq \frac { \partial_r (A , B)} { \lt{A}  \lt{B} }+\epsilon
\\[-1.2\baselineskip]
\end{align} 

with high probability for a large enough number of filters $N$.
This shows that a set of images consisting of  ``large patches'', meaning that $\frac { \partial_r (A , B)} { \lt{A} \lt{B} }$ is small (as in \exampleref{ex} for small $r$), is embedded almost-isometrically by a random ReLU convolutional layer. 
Moreover, \textit{any} set of images may be embedded almost-isometrically this way. To see how, fix $r$ while increasing artificially (digitally) the resolution of the images. The latter would yield $\frac { \partial_r (A , B)} { \lt{A}  \lt{B} } \rightarrow 0$ as the resolution tends to infinity.


\vspace{-1.\baselineskip}

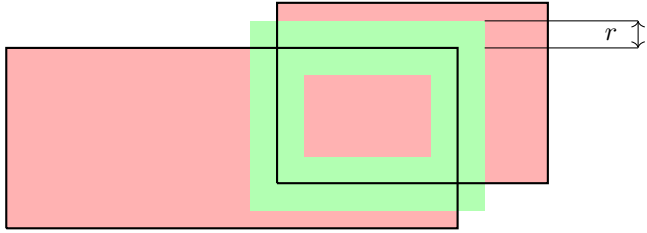
\begin{figure}\vspace*{-1cm}\centering
\begin{tikzpicture}[scale=1.2]
\fill [red!30] (-2, 1) -- (3, 1) -- (3, 3) -- (-2, 3) -- (-2, 1);
\fill [red!30] (1, 1.5) -- (4, 1.5) -- (4, 3.5) -- (1, 3.5) -- (1, 1.5);
\fill [green!30] (0.7, 3.3) -- (3.3, 3.3) -- (3.3, 1.2) -- (0.7, 1.2)
    -- (0.7, 3.3);
\fill [red!30] (1.3, 2.7) -- (2.7, 2.7) -- (2.7, 1.8) -- (1.3, 1.8)
    -- (1.3, 2.7);

\draw [thick] (-2, 1) -- (3, 1) -- (3, 3) -- (-2, 3) -- (-2, 1);
\draw [thick] (1, 1.5) -- (4, 1.5) -- (4, 3.5) -- (1, 3.5) -- (1, 1.5);
\draw [thin]  (3.3, 3.3) -- (5, 3.3);
\draw [thin]  (3.3, 3) -- (5, 3);
\draw [thin, arrows={<->}]  (5, 3) -- (5, 3.3);

\node at (4.7, 3.15) {$r$};
\end{tikzpicture}
\caption{Illustration of the shared $r$-boundary from 
\defref{def:boundary} in case of two axis-aligned rectangles.
The boundary is marked in green.}
\label{fig:boundary}\vspace{-1.3\baselineskip}
\end{figure}



\section{Discussion}\vspace*{-0.1cm}
\label{s:summary}

\vspace{-.2\baselineskip}
\subsection{Grayscale versus RGB Images}
\vspace{-.2\baselineskip}

The curious reader might wonder how the analysis over black-and-white (B/W) images can be extended to grayscale and RGB images. 
While we concentrated on B/W images in this work,
the definition of a shared boundary 
readily extends to grayscale images as does  the result of \thmref{thm:boundary}. 
This, in turn, explains the behavior of the grayscale-image dataset F-MNIST in \figref{fig:fmnist}.

For RGB images, the framework developed in this work does not hold in general.
To see why, consider two extremely simple RGB images: $Red$---consisting of all-zero green and blue channels and a constant unit-norm red channel, and $Green$---consisting of all-zero red and blue channels and a constant unit-norm green channel.
Clearly, $\inner{Red}{Green}=0$. However, a simple computation, 
in the spirit of this work, shows that $\E{ \inner{\ReLU(F*Red)}{\ReLU(F*Green) } } = 1/\pi = \R(0)$ which is far from $0$. That is, for such monochromatic pairs, the network behaves as a fully connected network \eqref{eq:rho_in-rho_out} and not according to the almost isometric behavior of \thmref{thm:boundary} and \figref{fig:imagenet}.

So why does an almost isometric relation appear in \twofigref{fig:cifar-10}{fig:imagenet}? The answer is somewhat surprising and of an independent interest. It appears that for the RGB datasets CIFAR-10 and ImageNet, the RGB channels have  high cosine similarity (which might not be surprising on its own), and additionally all the three channels have roughly the same norm. Quantitatively---averaged over $10^4$ images that were picked uniformly at random from the ImageNet dataset---the angle between two channels is \url{~}$12.5^{\circ}$ and the relative difference between the channel norms is \url{~}$9.2\%$. 

This phenomenon implies that our analysis for black and white images (as well as its extension to grayscale images) holds for RGB natural-images datasets. That said, there are images that are roughly monochromatic (namely, predominantly single-color images) but their measure is small and therefore these images can be treated as outliers. 


\vspace{-.2\baselineskip}
\subsection{Future Work} 
\vspace{-.2\baselineskip}
 
\Figref{fig:transform} demonstrates that a randomly initialized convolutional layer induces an almost-isometric yet contracting embedding of natural images (mind the linear fit with slope \url{~}0.97 in the figure), and therefore the contraction intensifies (``worsens'') with each additional layer.

In practice, the datasets are preprocessed by operations such as 
mean removal and normalization.
These operations induce a sharper contraction. Since a sharp collapse in the deeper layers implies failure to achieve dynamical isometry, a natural goal for a practitioner would be to prevent such a collapse. 
\citet{Saxe14exactsolutions} propose to replace  the i.i.d.\ Gaussian initialization with orthogonal initialization that is chosen uniformly over the orthogonal group for linear FNNs.
However, the latter would fail to prevent the collapse phenomenon as well (see \citep{pennington2017resurrecting}). This can be shown by a similar calculation to that of the dual activation of ReLU with i.i.d.\ Gaussian initialization. In fact, the dual activations of both initializations will coincide in the limit of large input dimensions. 

That said, the collapse can be prevented by moving further away from i.i.d.\ initialization. For example, by drawing the weights of every filter from the Gaussian distribution under the condition that the filter's sum of weights is zero.  Curiously, also including batch normalization  between the layers prevents such a collapse. In both cases, the opposite phenomenon to a collapse happens, an expansion. In  the deeper layers, the dataset  pre-activation vectors  become orthogonal to one another. Understanding this expansion and why it holds for seemingly unrelated modifications is a compelling line of research.







\subsubsection*{Acknowledgments}

The work of I.~Nachum and M.~Gastpar in this paper was supported in part by the Swiss National Science Foundation under Grant 200364.

The work of A.~Khina was supported by the \textsc{Israel Science Foundation} (grants No.\ 2077/20) and
the WIN Consortium through the Israel Ministry of Economy and Industry.





\bibliographystyle{unsrtnat}
\bibliography{mlbib}

\newpage
\appendix







\section{Proof of \thmref{thm:JL-general}}

To prove \thmref{thm:JL-general}, we first prove the following simple result.
\begin{lemma} 
\label{lem:Gaussian-vecs}
    Let $u,v\in \RR^k$ be two (deterministic) vectors, and let $G,H\in \RR^k$ be zero-mean jointly Gaussian vectors with the following auto- and cross-correlation matrices:
    \begin{align}
        \E{G G^T} &= \E{H H^T} = \nu^2 I_k,
      & \E{G H^T} &= \rho \nu^2 I_k.
    \end{align}
    Then, 
    \begin{align} 
        Z 
        = \begin{pmatrix} 
            \inner{u}{G} 
         \\ \inner{v}{H}
          \end{pmatrix}
        = \begin{pmatrix} 
            u^T G 
         \\ v^T H
          \end{pmatrix} 
    \end{align}
    is a zero-mean Gaussian vector with a covariance matrix 
    \begin{align}
        \E{Z Z^T} &= 
        \nu^2 
        \begin{pmatrix} 
            \norm{u}^2 & \rho \inner{u}{v} 
         \\ \rho \inner{u}{v} & \norm{v}^2 
        \end{pmatrix}
        .
    \end{align} 
\end{lemma} 

\begin{proof}[Proof of \lemref{lem:Gaussian-vecs}]
    $Z$ is a Gaussian vector as a linear combination of two jointly Gaussian vectors. Furthermore, 
    \begin{align}
        \E{Z} &= \E{ 
          \begin{pmatrix} 
            u^T G 
         \\ v^T H
          \end{pmatrix} 
        }
        = \begin{pmatrix} 
            u^T \E{G} 
         \\ v^T \E{H} 
          \end{pmatrix} 
        = 0, 
     \\ \E{Z Z^T} &= \E{ 
        \begin{pmatrix} 
            u^T G G^T u & u^T G H^T v
         \\ v^T H G^T u & v^T H H^T v
        \end{pmatrix} 
        } 
        = 
        \begin{pmatrix} 
            u^T \E{G G^T} u & u^T \E{G H^T} v
         \\ v^T \E{H G^T} u & v^T \E{H H^T} v
        \end{pmatrix} 
     \\ &= 
        \nu^2 
        \begin{pmatrix} 
            \norm{u}^2 & \rho \inner{u}{v} 
         \\ \rho \inner{u}{v} & \norm{v}^2 
        \end{pmatrix}\qedhere
    \end{align}
\end{proof}

We are now ready to prove \thmref{thm:JL-general}.
\begin{proof}[Proof of \thmref{thm:JL-general}]
    The equation~\eqref{eq:10} is proved as follows.
    \begin{subequations}
    \label{eq:proof:dual-activation}
    \begin{align} 
        \E{ \inner{\sigma (F * x)}{\sigma( F * y )} } 
        &=  \sum_{i,j \in \sZ_n} \E{ \sigma (\inner{F}{[x]^r_{ij}}) \cdot \sigma \left( \inner{F}{[y]^r_{ij}} \right) }
    \label{eq:proof:dual-activation:convolution}
     \\ &= \sum_{i,j \in \sZ_n} \dualactthree \left(  [x]^r_{ij}, [y]^r_{ij},\nu \right) ,
    \label{eq:proof:dual-activation:def}
    \end{align} 
    \end{subequations}
    where \eqref{eq:proof:dual-activation:convolution} follows from the definition of the cyclic convolution \eqref{eq:def:convolution} and the linearity of expectation, 
    and \eqref{eq:proof:dual-activation:def} follows from \defref{def:dual-activation} using \lemref{lem:Gaussian-vecs} with $F$ taking the role of $G$ and $H$ (with $\rho=1)$, 
    and $[x]^r_{ij}$ and $[y]^r_{ij}$ taking the roles of $u$ and $v$, respectively.
    
    Now, we set out to prove~\eqref{eq:07} under the assumption that $\sigma$ is homogeneous and the following notation.
    \begin{align}
    \begin{pmatrix} X_{ij} \\ Y_{ij}  \end{pmatrix} 
    &\sim \gN\left( \begin{pmatrix} 0 \\ 0 \end{pmatrix}, 
         \begin{pmatrix}
            \lv [x]^r_{ij} \rv ^2 &  \inner{[x]^r_{ij}}{[y]^r_{ij}} 
         \\[3pt] 
            \inner{[x]^r_{ij}}{[y]^r_{ij}} &  \lv [y]^r_{ij} \rv ^2 
         \end{pmatrix}
         \right)\\
        \begin{pmatrix} \oX_{ij} \\ \oY_{ij}  \end{pmatrix} 
        &\sim \gN\left( \begin{pmatrix} 0 \\ 0 \end{pmatrix}, 
         \begin{pmatrix}
            1 &  \rho_{ij}
         \\[3pt] 
           \rho_{ij}  & 1 
         \end{pmatrix}
         \right)
     \end{align}
 
   \begin{subequations}
    \label{eq:proof:dual-activation2}
    \begin{align} 
        \E{ \inner{\sigma (F * x)}{\sigma( F * y )} }    &= \sum_{i,j \in \sZ_n} \dualactthree \left(  [x]^r_{ij}, [y]^r_{ij},\nu \right)    
    \\  \label{eq:jl1} 
       &=  \sum_{i,j \in \sZ_n}\E{\sigma (\nu X_{ij})  \cdot \sigma (\nu Y_{ij})} \\  \label{eq:jl2} 
        &= \nu^2 \sum_{i,j \in \sZ_n} \lt{[x]^r_{ij}}\lt{[y]^r_{ij} } \E{ \sigma (X_{ij}/\lt{[x]^r_{ij}  }) \cdot \sigma ( Y_{ij}/\lt{[y]^r_{ij}  }) } \\   \label{eq:jl3} 
        &=  \nu^2 \sum_{i,j \in \sZ_n} \lt{[x]^r_{ij}}\lt{[y]^r_{ij} } \E{ \sigma (\oX_{ij}  ) \cdot \sigma (\oY_{ij} ) } \\ \label{eq:jl4} 
        &=  \nu ^2\cdot\nu^2_{\sigma} \sum_{i,j \in \sZ_n} \lt{ [x]^r_{ij}}  \lt{[y]^r_{ij}} \dualactone \left(  \rho_{ij}  \right),
         \end{align} 
    \end{subequations}
   where \eqref{eq:jl1} follows from \defref{def:dual-activation}, \eqref{eq:jl2} follows from the homogeneity of $\sigma$, \eqref{eq:jl3} follows from standard random variable calculus, and \eqref{eq:jl4} follows from~\eqref{eq:def:dual-activation}.
   
   Finally, to obtain~\eqref{eq:08}, we apply~\eqref{eq:07}
   for $x=y$ and use that $\rho_{ij}=1$ and
   $\dualactone(1)=1$:
   \begin{align}
       \E{\sigma(F*x)^2}
       &=\nu^2\cdot\nu_\sigma^2
       \sum_{i,j\in\mathbb{Z}_n}\|[x]_{ij}^r\|^2
       \dualactone(\rho_{ij})\\
       &=\nu^2\cdot\nu_\sigma^2
       \sum_{i,j\in\mathbb{Z}_n}\|[x]_{ij}^r\|^2
       \dualactone(1)\\
       &=\nu^2\cdot\nu_\sigma^2
       \sum_{i,j\in\mathbb{Z}_n}\|[x]_{ij}^r\|^2
       =\nu^2\cdot\nu_\sigma^2\cdot r^2
       \|x\|^2\;.\qedhere
   \end{align}
\end{proof}


\section{Proof of \thmref{thm:JL-general:CoM}}

To prove the concentration of measure results of this work, we will make use of the following definition.
\begin{definition}
\label{def:Orlicz-norm}
    The Orlicz norm of a random variable (RV) $X$ with respect to (w.r.t.) a convex function $\psi:[0,\infty) \rightarrow [0,\infty)$ such that $\psi(0)=0$ and $\lim_{x \rightarrow \infty} \psi(x)=\infty$ is defined by 
    \begin{align}
        \norm{X}_{\psi} := \inf \left\{ t > 0 \ \middle|\ \E{ \psi \left( \frac{|X|}{t} \right)} \leq 1  \right\}.
    \end{align}
    In particular, $X$ is said to be \textit{sub-Gaussian} if $\pt{X} < \infty$, and \textit{sub-exponential} if $\po{X} < \infty$, where 
    $\psi_p(x) := \exp\left\{ x^p \right\} - 1$ for $p \geq 1$.
\end{definition}


The following two results, whose proofs are available in \cite[Ch.~2 and~5.2]{HDP}, will be useful for the proof of \thmref{thm:JL-general:CoM}.

\begin{lemma}
\label{lem:prod}
    Let $X$ and $Y$ be sub-Gaussian random variables (not necessarily independent). 
    Then, 
    \begin{enumerate}
    \item \textbf{Sum of independent sub-Gaussians.}
        If $X$ and $Y$ are also independent, then their sum, $X+Y$, is sub-Gaussian. 
        Moreover,
        $\pt{X + Y}^2 \leq C \left( \pt{X}^2 + \pt{Y}^2 \right)$ 
        for an absolute constant $C$. 
        The same holds (with the same constant $C$) also for sums of multiple independent sub-Gaussian RVs.
    \item 
        \textbf{Centering.}
        $X - \E{X}$ is sub-Gaussian. 
        Moreover, 
        $\pt{X - \E{X}} \leq C \pt{X}$ for an absolute constant $C$.
    \item
        \textbf{Lipschitz functions of Gaussians.}
        If $X$ is a centered Gaussian and $f$ is a function with Lipschitz
        constant $L$,
        then $\pt{f(X)-\E{f(X)}}\le C\pt{X}$ for an absolute constant $C$.
    \item 
        \textbf{Product of sub-Gaussians.}
        $XY$ is sub-exponential. 
        Moreover,
        $\lv XY \rv_{\psi_1} \leq  \lv X \rv_{\psi_2} \lv Y \rv _{\psi_2}$.
    \end{enumerate}
\end{lemma}
 
\begin{theorem}[Bernstein's inequality for sub-exponentials]
\label{thm:Bernstein}
    Let $X_1, \ldots , X_N$ be independent zero-mean sub-exponential RVs. Then, 
    \begin{align} 
        \sP \left( \left| \frac{1}{N} \sum_{i=1}^N X_i  \right|    \geq t \right) &\leq 2\exp \left\{ - \min \left\{ \frac{t^2}{K^2}, \frac{t}{K} \right\} \cdot c \cdot N \right\},
        & \forall t \geq 0,
    \end{align} 
    where $K=\max_i \po{X_i}$ and $c > 0$ is an absolute constant. 
\end{theorem}

We are now ready to prove the desired concentration-of-measure result.
\begin{proof}[Proof of \thmref{thm:JL-general:CoM}]
    $\inner{F_\ell}{[x]^r_{ij}}$ and $\inner{F_\ell}{[y]^r_{ij}}$ are (jointly) Gaussian RVs---as linear combinations of i.i.d.\ Gaussian RVs---with mean zero and variances $\nu^2 \norm{[x]^r_{ij}}^2$ and $\nu^2 \norm{[y]^r_{ij}}^2$, respectively,
    for all $\ell \in [N]$ and $i,j \in \sZ_n$.
    Hence, $\inner{F_\ell}{[x]^r_{ij}}$ and $\inner{F_\ell}{[y]^r_{ij}}$ are also sub-Gaussian with
    \begin{align}
        \pt{\inner{F_\ell}{[x]^r_{ij}}} &\le C_0\nu R,
      & \pt{\inner{F_\ell}{[y]^r_{ij}}} &\le C_0\nu R,
      & \forall \ell &\in [N], 
      & \forall i,j &\in \sZ_n 
    \end{align} 
    
    
    for some universal constant $C_0>0$. Define now
    \begin{align}
        X_{ij\ell} &:= \sigma \left( \inner{F_\ell}{[x]^r_{ij}} \right),
      & Y_{ij\ell} &:= \sigma \left( \inner{F_\ell}{[y]^r_{ij}} \right),
      & \forall \ell &\in [N], 
      & \forall i,j &\in \sZ_n .
    \end{align}
    Since $\sigma$ is Lipschitz continuous with a Lipschitz constant $L$, 
    and since the inner products $\langle F_\ell, [x]^r_{ij}\rangle$ and $\langle F_\ell, [y]^r_{ij}\rangle$
    are centered Gaussians, by property~3 in \lemref{lem:prod}
    it follows that
    $X_{ij\ell}$ and $Y_{ij\ell}$ are sub-Gaussian with 
    \begin{align} 
    \label{eq:def:Xijl}
        \pt{X_{ij\ell}-\E{X_{ij\ell}}} &\leq C_0 L \nu R,
      & \pt{Y_{ij\ell}-\E{Y_{ij\ell}}} &\leq C_0 L \nu R,
      & \forall \ell &\in [N], 
      & \forall i,j &\in \sZ_n .\ \ 
    \end{align}
    Due to the assumption $\sigma(0)=0$ and the fact that $\sigma$
    has Lipschitz constant $L$, inequality 
    $|\sigma(x)|\le L|x|$ holds for every $x\in\mathbb{R}$.
    Accordingly, we can bound the expectation of $\E{ X_{ij\ell}}$
    with
    \begin{align}\label{eq:13}
        \left|\E{X_{ij\ell}}\right|
        &\le\E{\left| X_{ij\ell} \right|}
        \le L\E{ \left|\inner{F_\ell}{[x]_{ij}^r}\right|}
        \le L\sqrt{\Var\big[\inner{F_\ell}{[x]_{ij}^r}\big]}
        \le L\nu R\;,
    \end{align}
    where we applied the Cauchy--Schwarz inequality. Similarly, it holds
    that
    \begin{align}\label{eq:14}
        \big|\E{Y_{ij\ell}}\big|\le L\nu R\;.
    \end{align}
    Since by homogeneity of norm it also holds for any constant 
    $D\in\mathbb{R}$ that
    $\pt{D}\le C_0|D|$, by triangle inequality it then follows
    from~\eqref{eq:def:Xijl}, \eqref{eq:13} and~\eqref{eq:14} that
    \begin{align}
        \pt{X_{ij\ell}}
        &\le\pt{X_{ij\ell}-\E{X_{ij\ell}}}
        +\pt{\E{X_{ij\ell}}} \le 2C_0L\nu R\;,\\
        \pt{Y_{ij\ell}}
        &\le 2C_0L\nu R\;.
    \end{align}
    Therefore, subsequent application of properties 4 and 2 of \lemref{lem:prod} to $X_{ij\ell} Y_{ij\ell}$, yields
    \begin{align}
    \label{eq:Xijl:Orlicz-norm}
    \begin{aligned} 
        \po { X_{ij\ell} Y_{ij\ell} } &\leq 4C_0^2 \nu^2 L^2 R^2, 
     \\ \po { X_{ij\ell} Y_{ij\ell} - \E{X_{ij\ell} Y_{ij\ell}} } 
     &
     \leq 4C_0^2 C \nu^2 L^2 R^2 ,
    \end{aligned}
    \end{align} 
    for all $\ell \in [N]$ and all $i,j \in \sZ_n$,
    where $C>0$ is is the absolute constant from \lemref{lem:prod}. Now it follows
    \begin{subequations} 
    \label{eq:proof:general:CoM}
    \noeqref{eq:proof:general:CoM:max}
    \begin{align}
        \sP &\left( \left| \frac{1}{N} \sum_{\ell=1}^N \inner{\sigma (F_\ell * x)}{\sigma( F_\ell * y )} - \sum_{i,j=1}^n \dualactthree\left(  [x]^r_{ij},[y]^r_{ij},\nu \right)\right| \geq  \epsilon \right) 
     \\ &\qquad = \sP \left( \frac{1}{N} \left| \sum_{\ell\in [N]; i,j \in \sZ_n} \Big\{ X_{ij\ell} Y_{ij\ell} - \E{ X_{ij\ell} Y_{ij\ell} } \Big\} \right| \geq \eps \right)
    \label{eq:proof:general:CoM:dual-activation}
     \\ &\qquad \leq \sP \left(\frac{1}{N} \sum_{i,j \in \sZ_n} \left| \sum_{\ell\in [N]} \Big\{ X_{ij\ell} Y_{ij\ell}
        - \E{ X_{ij\ell}Y_{ij\ell} } \Big\} \right| \geq \eps \right)
    \label{eq:proof:general:CoM:triang-ineq}
     \\ &\qquad \leq \sP \left(\frac{n^2}{N} \max_{i,j \in \sZ_n} \left| \sum_{\ell\in [N]} \Big\{ X_{ij\ell} Y_{ij\ell} - \E{ X_{ij\ell}Y_{ij\ell} } \Big\} \right| \geq \eps \right)
    \label{eq:proof:general:CoM:max}
     \\ &\qquad \leq \sum_{i,j \in \sZ_n} \sP \left(\frac{n^2}{N} \left| \sum_{\ell\in [N]} \Big\{ X_{ij\ell}Y_{ij\ell} - \E{ X_{ij\ell}Y_{ij\ell} } \Big\} \right| \geq \eps \right)
    \label{eq:proof:general:CoM:UB}
     \\ &\qquad \leq 2 n^2 \exp \left\{ -\min \left( \frac{16\eps^2}{C_0^4C^2 \nu^4 L^4 R^4 n^4}, \frac{4\eps}{C_0^2 C \nu^2 L^2 R^2 n^2} \right) c N \right\} ,
    \label{eq:proof:general:CoM:Bernstein}
    \end{align} 
    \end{subequations}
    where 
    \eqref{eq:proof:general:CoM:dual-activation} follows from \thmref{thm:JL-general} and \eqref{eq:def:Xijl},
    \eqref{eq:proof:general:CoM:triang-ineq} follows from the triangle inequality,
    \eqref{eq:proof:general:CoM:UB} follows from the union bound, 
    and \eqref{eq:proof:general:CoM:Bernstein} follows from \eqref{eq:Xijl:Orlicz-norm} and \thmref{thm:Bernstein}.
    

    Finally, noting that \eqref{eq:thm:CoM} holds for 
    $N >  \max \left(   \frac{C_0^4  C^2\nu^4L^4R^4 n^4}{16 c \epsilon^2},  
    \frac{C_0^2 C\nu^2L^2R^2 n^2}{4 c \epsilon} \right) \log \frac{2 n^2}{\delta}$ concludes the proof.
\end{proof}


\section{Proof of \corref{cor:similarity}}

Let us start by inspecting~\eqref{eq:11}
to observe that, by homogeneity of $\sigma$, 
neither the
value of~$\brho_\OUT$ nor the distribution of~$\varrho_\OUT$
depend on $\nu$ or the norms of $x$ and $y$.
Therefore, let us assume without loss of generality that
$x$ and $y$ are unit vectors, as well as that
$\nu=\frac{1}{\nu_\sigma r}$.
Furthermore, note that homogeneity of $\sigma$ implies $\sigma(0)=0$,
so \thmref{thm:JL-general:CoM} can be applied with activation function $\sigma$.

Let $A:=\frac{1}{N}\sum_{\ell=1}^N
\inner{\sigma(F_\ell*x)}{\sigma(F_\ell*y)}$,
$B:=\frac{1}{N}\sum_{\ell=1}^N\|\sigma(F_\ell*x)\|^2$,
$C:=\frac{1}{N}\sum_{\ell=1}^N\|\sigma(F_\ell*y)\|^2$
and $A':=\frac{1}{r^2}\sum_{i,j\in\mathbb{Z}_n}
\|[x]_{ij}^r\|\cdot\|[y]_{ij}^r\|\dualactone(\rho_{ij})$.
Comparing against~\eqref{eq:11}, we see that $\varrho_\OUT=\frac{A}{\sqrt{BC}}$ and
that $\brho_\OUT=A'$.

We now apply \thmref{thm:JL-general:CoM} three times, for pairs of
vectors $(x,y)$, $(x,x)$ and $(y,y)$, respectively, with parameters
of $\epsilon/5$ and $\delta/3$. Using equations~\eqref{eq:10}--\eqref{eq:08}
from \thmref{thm:JL-general}, we
check that indeed each of the three following
relations holds
for $N>\max(K,K^2)\log\frac{6n^2}{\delta}$, 
except with probability $\delta/3$: 
\begin{align}\label{eq:12}
    |A-A'|\le\frac{\epsilon}{5}\;,
    ~~~~~~~~
    |B-1|\le\frac{\epsilon}{5}\;,
    ~~~~~~~~
    |C-1|\le\frac{\epsilon}{5}\;.
\end{align}
By union bound, all three inequalities in~\eqref{eq:12}
hold simultaneously except with probability $\delta$.
Finally, by elementary inequalities we establish that
if~\eqref{eq:12} holds for some $A,A',B,C\in\mathbb{R}$ such
that $|A|\le 1$ and $0<\epsilon\le1/10$, then it also holds
that 
\begin{align}
    \big|\varrho_\OUT-\brho_\OUT\big|&=
    \left|\frac{A}{\sqrt{BC}}-A'\right|\le\epsilon\;.
\end{align}
Since this occurs except with probability $\delta$, the
proof is concluded.\hfill\qedsymbol


\section{Proof of lemmata \ref{lem:JL} and \ref{lem:JL:CoM}}

\begin{proof}[Proof of \lemref{lem:JL}]
    \Lemref{lem:JL} may be viewed as a special case of \thmref{thm:JL-general} with an identity activation function $\sigma$. Therefore, 
    \begin{align}
        \E{\inner{F * x}{F * y}} 
        \stackrel{(a)}= \sum_{i,j \in \sZ_n} \dualactthree \left(  [x]^r_{ij},[y]^r_{ij},\nu \right)
        \stackrel{(b)}= \sum_{i,j \in \sZ_n} \nu^2 \inner{[x]^r_{ij}}{[y]^r_{ij}}
        \stackrel{(c)}= \inner{x}{y},
    \end{align}
    where $(a)$ follows from \thmref{thm:JL-general}, 
    $(b)$ follows from \defref{def:dual-activation},
    and $(c)$ holds since $\nu^2 = 1/r^2$.
\end{proof}

\begin{proof}[Proof of \lemref{lem:JL:CoM}]
    Again, \lemref{lem:JL:CoM} follows from \thmref{thm:JL-general:CoM} for
    $\sigma(x)\equiv x$ (and hence $L=1$) and from \lemref{lem:JL}.
\end{proof}


\section{Proof of \thmref{thm:con_map}}

    We start with the lower bound.
    Define $\rho_{ij} := \frac{\inner{[x]^r_{ij}}{[y]^r_{ij}}}{\norm{[x]_{ij}^r} \norm{[y]_{ij}^r}}$
    for all $i,j \in \sZ_n$. Then, for all $i,j \in \sZ_n$, 
    \begin{subequations} 
    \label{eq:proof:bounds:LBij}
    \begin{align}
        \inner{[x]^r_{ij}}{[y]^r_{ij}} &=   \lt{[x]^r_{ij}} \lt{[y]^r_{ij}}  \rho_{ij}
    \label{eq:proof:bounds:LBij:rho_ij}
     \\ & \leq \lt{[x]^r_{ij}} \lt{[y]^r_{ij}}  \frac{\sqrt{1-\rho_{ij}^2} +(\pi -\cos^{-1}(\rho_{ij}))\rho_{ij} }{\pi}
    \label{eq:proof:bounds:LBij:inequality}
     \\ &=r^2 \E{ \sigma \left( \inner{F}{[x]^r_{ij}} \right) \sigma \left( \inner{F}{[y]^r_{ij}} \right)},
    \label{eq:proof:bounds:LBij:non-linearity}
    \end{align}
    \end{subequations}
    where \eqref{eq:proof:bounds:LBij:rho_ij} holds by the definition of $\rho_{ij}$,
    \eqref{eq:proof:bounds:LBij:inequality} holds since 
    \begin{align} 
        a&\leq\frac{\sqrt{1-a^2} +\left( \pi -\cos^{-1}(a) \right)a }{\pi}  & \forall |a| \leq 1 , 
    \end{align} 
    and \eqref{eq:proof:bounds:LBij:non-linearity} follows from
    \citep[Table~1 and Section~C of supplement]{daniely}. 
%
%
%
    Thus, 
    \begin{align}\label{eq:01}
        \inner{x}{y} 
        = 1/r^2 \sum_{i,j \in \sZ_n} \inner{[x]_{ij}^r}{[y]_{ij}^r}
        \leq \sum_{i,j \in \sZ_n} \E{ \sigma \left( \inner{F}{[x]^r_{ij}} \right) \sigma \left( \inner{F}{[y]^r_{ij}} \right)}
        = \E{ \sigma (F * x )\sigma( F * y ) } ,
    \end{align}
    where the inequality follows from \eqref{eq:proof:bounds:LBij}, and the last step is due to \eqref{eq:proof:dual-activation}.
    Now, since the ReLU activation is non-negative, $0 \leq \E{ \sigma ( F * x )\sigma( F * y )}$, which completes the proof  of the left inequality in \eqref{eq:thm:con_map:bounds}.
%

For the upper bound, we use the following convexity argument.
First, by homogeneity of ReLU, we can assume without loss
of generality that $x$ and $y$ have unit norm. 
Therefore, it remains to
show $\E{\sigma(F*x)\sigma(F*y)}\le\frac{1+\rho}{2}$ for
$x,y$ such that $\|x\|=\|y\|=1$ and $\inner{x}{y}=\rho$.

Recall that $\rho_{ij}=\frac{\inner{[x]_{ij}^r}{[y]_{ij}^r}}
{\|[x]_{ij}^r\|\|[y]_{ij}^r\|}$.
We expand the definitions in the same way as in~\eqref{eq:01}
and~\eqref{eq:proof:bounds:LBij:non-linearity}
\begin{align}
\E{\sigma(F*x)\sigma(F*y)}
&=\sum_{i,j\in\mathbb{Z}_n}
\E{\sigma\left(\inner{F}{[x]_{ij}^r}\right)
\sigma\left(\inner{F}{[y]_{ij}^r}\right)}\\
&=\sum_{i,j\in\mathbb{Z}_n}
\frac{\|[x]_{ij}^r\|\|[y]_{ij}^r\|}{r^2}\Rone(\rho_{ij})\;.
\label{eq:02}
\end{align}
It is easily checked that $\Rone(1)=1$, $\Rone(-1)=0$ and
that $\Rone(x)$ is convex for $-1\le x\le 1$.
Accordingly, using the decomposition
$\rho=\frac{1+\rho}{2}\cdot 1+\frac{1-\rho}{2}\cdot(-1)$,
we have by Jensen's inequality for every
$i,j\in\mathbb{Z}_n$
\begin{align}
    \Rone(\rho_{ij})&\le\frac{1+\rho_{ij}}{2}\;.
\end{align}
Substituting into~\eqref{eq:02},
\begin{align}
    \E{\sigma(F*x)\sigma(F*y)}
    &\le\sum_{i,j\in\mathbb{Z}_n}
    \frac{\|[x]_{ij}^r\|\|[y]_{ij}^r\|}{r^2}\cdot
    \frac{1+\rho_{ij}}{2}\\
    &=\frac{1}{2r^2}\sum_{i,j\in\mathbb{Z}_n}
    \|[x]_{ij}^r\|\|[y]_{ij}^r\|
    +\frac{1}{2r^2}\sum_{i,j\in\mathbb{Z}_n}
    \inner{[x]_{ij}^r}{[y]_{ij}^r}
    \\
    &\le\frac{1}{2r^2}\sqrt{
    \left(\sum_{i,j\in\mathbb{Z}_n}\|[x]_{ij}^r\|^2\right)
    \left(\sum_{i,j\in\mathbb{Z}_n}\|[y]_{ij}^r\|^2\right)
    }
    +\frac{1}{2}\cdot\inner{x}{y}
    \label{eq:03}
    \\
    &=\frac{1+\rho}{2}\;,
\end{align}
where in~\eqref{eq:03} we applied the Cauchy--Schwarz inequality
and used the initial assumption $\|x\|=\|y\|=1$.
\hfill \qedsymbol


\begin{figure}[t]\centering
\begin{tikzpicture}
\fill[red!30] (0, 0) -- (1, 0) -- (1, 1.5) -- (0, 1.5)
-- (0, 0);
\fill[red!60] (0, 1.5) -- (1, 1.5) -- (1, 3) -- (0, 3)
-- (0, 1.5);
\fill[blue!30] (1.5, 0) -- (2.5, 0) -- (2.5, 1.5) -- (1.5, 1.5) -- (1.5, 0);
\fill[red!60] (1.5, 1.5) -- (2.5, 1.5) -- (2.5, 3) -- (1.5, 3)
-- (1.5, 1.5);

\node at (0.5, 2.25) {$\frac{1+\rho}{2}$};
\node at (0.5, 0.75) {$\frac{1-\rho}{2}$};
\node at (2, 2.25) {$\frac{1+\rho}{2}$};
\node at (2, 0.75) {$-\frac{1-\rho}{2}$};

\draw[thick] (0, 0) -- (1, 0) -- (1, 1.5) -- (0, 1.5)
-- (0, 0);
\draw[thick] (0, 0) -- (1, 0) -- (1, 3) -- (0, 3)
-- (0, 0);
\node at (0.5, 3.5) {$x$};
\draw[thick] (1.5, 0) -- (2.5, 0) -- (2.5, 1.5) -- (1.5, 1.5) -- (1.5, 0);
\draw[thick] (1.5, 0) -- (2.5, 0) -- (2.5, 3) -- (1.5, 3)
-- (1.5, 0);
\node at (2, 3.5) {$y$};

\node at (1.25, -0.5) {$\brho_{\OUT}=\frac{1+\rho}{2}$};
\fill[red!60] (3.85, 2) -- (5.15, 2) -- (5.15, 3) -- (3.85, 3) -- (3.85, 2);
\fill[red!30] (3.85, 1) -- (5.15, 1) -- (5.15, 2) -- (3.85, 2) -- (3.85, 1);
\fill[blue!30] (5.65, 1) -- (6.95, 1) -- (6.95, 2) -- (5.65, 2) -- (5.65, 1);
\fill[red!60] (5.65, 0) -- (6.95, 0) -- (6.95, 1) -- (5.65, 1) -- (5.65, 0);

\node at (4.5, 2.5) {$\sqrt{1-\rho}$};
\node at (4.5, 1.5) {$\sqrt{\rho}$};
\node at (4.5, 0.5) {$0$};
\node at (6.3, 2.5) {$0$};
\node at (6.3, 1.5) {$-\sqrt{\rho}$};
\node at (6.3, 0.5) {$\sqrt{1-\rho}$};

\draw[thick] (3.85, 2) -- (5.15, 2) -- (5.15, 3) -- (3.85, 3) -- (3.85, 2);
\draw[thick] (3.85, 1) -- (5.15, 1) -- (5.15, 2) -- (3.85, 2) -- (3.85, 1);
\draw[thick] (3.85, 0) -- (5.15, 0) -- (5.15, 1) -- (3.85, 1) -- (3.85, 0);
\node at (4.5, 3.5) {$x$};
\draw[thick] (5.65, 2) -- (6.95, 2) -- (6.95, 3) -- (5.65, 3) -- (5.65, 2);
\draw[thick] (5.65, 1) -- (6.95, 1) -- (6.95, 2) -- (5.65, 2) -- (5.65, 1);
\draw[thick] (5.65, 0) -- (6.95, 0) -- (6.95, 1) -- (5.65, 1) -- (5.65, 0);
\node at (6.3, 3.5) {$y$};
\node at (5.4, -0.5) {$\brho_{\OUT}=0$};
\fill[red!60] (8.35, 2) -- (9.65, 2) -- (9.65, 3) -- (8.35, 3) -- (8.35, 2);
\fill[red!30] (8.35, 1) -- (9.65, 1) -- (9.65, 2) -- (8.35, 2) -- (8.35, 1);
\fill[red!30] (10.15, 1) -- (11.45, 1) -- (11.45, 2) -- (10.15, 2) -- (10.15, 1);
\fill[red!60] (10.15, 0) -- (11.45, 0) -- (11.45, 1) -- (10.15, 1) -- (10.15, 0);

\node at (9, 2.5) {$\sqrt{1-\rho}$};
\node at (9, 1.5) {$\sqrt{\rho}$};
\node at (9, 0.5) {$0$};
\node at (10.8, 2.5) {$0$};
\node at (10.8, 1.5) {$\sqrt{\rho}$};
\node at (10.8, 0.5) {$\sqrt{1-\rho}$};

\draw[thick] (8.35, 2) -- (9.65, 2) -- (9.65, 3) -- (8.35, 3) -- (8.35, 2);
\draw[thick] (8.35, 1) -- (9.65, 1) -- (9.65, 2) -- (8.35, 2) -- (8.35, 1);
\draw[thick] (8.35, 0) -- (9.65, 0) -- (9.65, 1) -- (8.35, 1) -- (8.35, 0);
\node at (9, 3.5) {$x$};
\draw[thick] (10.15, 2) -- (11.45, 2) -- (11.45, 3) -- (10.15, 3) -- (10.15, 2);
\draw[thick] (10.15, 1) -- (11.45, 1) -- (11.45, 2) -- (10.15, 2) -- (10.15, 1);
\draw[thick] (10.15, 0) -- (11.45, 0) -- (11.45, 1) -- (10.15, 1) -- (10.15, 0);
\node at (10.8, 3.5) {$y$};
\node at (9.9, -0.5) {$\brho_{\OUT}=\rho$};
\end{tikzpicture}
\caption{ An illustration of the examples from \remref{rem:tightness}.
Note that similar examples can be constructed in higher dimensions, e.g.,
by equally distributing the coefficients over a larger set of coordinates.
}
\label{fig:tightness}
\end{figure}
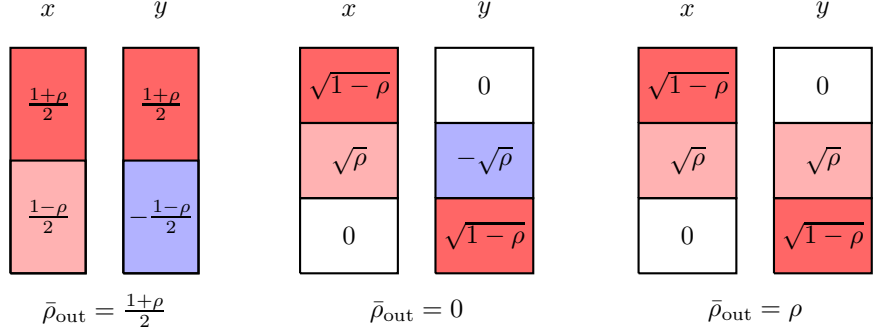

\section{Proof of \thmref{thm:gaussian}}

    The proof follows from the following steps.
    \begin{subequations}
    \label{eq:proof:Gaussian-input}
    \begin{align} 
        \E{ \inner{\sigma(F*X)}{\sigma(F * Y)} }
        &= \E{\CE{\sum_{i,j \in \sZ_n} \sigma(\inner{F}{[X]_{ij}^r})\sigma(\inner{F}{[Y]_{ij}^r})}{F}}
    \label{eq:proof:Gaussian-input:smoothing}
     \\ &= \E{ n^2 \E{\sigma \left( X_F \right) \sigma \left( Y_F \right)} }
    \label{eq:proof:Gaussian-input:symmetry}
     \\ &= \frac{1}{2} \frac{\sqrt{1 - \rho^2} + \left( \pi - \cos^{-1}(\rho) \right) \rho}{\pi} \E{ \norm{F}^2 } 
    \label{eq:proof:Gaussian-input:correlated-sigmas}
     \\ &= \frac{\sqrt{1 - \rho^2} + \left( \pi - \cos^{-1}(\rho) \right) \rho}{\pi}
    \label{eq:proof:Gaussian-input:F-var}
    \end{align} 
    \end{subequations}
    where \eqref{eq:proof:Gaussian-input:smoothing} follows from the law of total expectation;
    \eqref{eq:proof:Gaussian-input:symmetry} follows from \eqref{eq:thm:gaussian-inputs:correlations}, from the independence of $F$ in $(X,Y)$, and from \lemref{lem:Gaussian-vecs} with $F$ taking the roles of $u$ and $v$, and $[X]_{ij}^r$ and $[Y]_{ij}^r$ taking the roles of $G$ and $H$, resulting in $X_F$ and $Y_F$ which are zero-mean jointly Gaussian with variances $\norm{F}^2 / n^2$ and a Pearson correlation coefficient $\rho$;
    \eqref{eq:proof:Gaussian-input:correlated-sigmas} follows from the homogeneity of the ReLU activation function and from \citep[eq.~(6)]{deep_ker}, \citep[Theorem~4]{giryes2016deep}, \citep[Section~8]{daniely}; 
    \eqref{eq:proof:Gaussian-input:F-var} holds by the definition of $F$.
    %
\hfill \qedsymbol


\section{Proof of \thmref{thm:boundary}}

In this section, we will use the following additional notations.
$[A]^{\Or}_{ij}$ will denote the submatrix of $A$ formed from rows $i-r$ to $i+r$ and columns $j-r$ to $j+r$, 
and $\mathbf{1}_{k\times k}$  will denote a $k\times k$ matrix with all entries equal $1$.

    The main effort in the proof is establishing~\eqref{eq:04}.
    We apply \thmref{thm:JL-general}, so we have  $ \E{ \inner{\sigma (F * A)}{\sigma( F * B )}} = 
   \sum_{i,j} \dualactthree(  [A]^{\Or}_{ij}, [B]^{\Or}_{ij} ,2/(2r+1) )$. And specifically, for the ReLU activation, it holds that  
     \begin{align}
     \label{eq1}
       \dualactthree(  [A]^{\Or}_{ij}, [B]^{\Or}_{ij} ,2/(2r+1) ) =  \frac { \lt{[A]^{\Or}_{ij}} \lt{[B]^{\Or}_{ij}} \R(\rho_{ij}) } { (2r+1)^2 }
   \end{align}
   where $\rho_{ij}= 
   \frac{ \inner{ [A]^{\Or}_{ij}   }{ [B]^{\Or}_{ij}  }  } {  \lt{ [A]^{\Or}_{ij}}  \lt{ [B]^{\Or}_{ij}} } $.
   
   A natural way to show that the convolutional layer induces  isometry in expectation would be to show 
   \begin{equation}\label{eq2}
       \dualactthree(  [A]^{\Or}_{ij}, [B]^{\Or}_{ij} ,2/(2r+1) )= A_{ij} \cdot B_{ij} 
   \end{equation}
   for all pairs $(i,j)$. Unfortunately, this does not hold for all pairs. 
   However, we will show that the only pairs for which
   it does not hold belong to the $r$-boundary
   $\partial_r(A,B)$.
   There are two cases to consider $A_{ij} \cdot B_{ij}=0$ and $A_{ij} \cdot B_{ij}=1$ since $A$ and $B$ are binary.
   
   Let $\dualactthree_{ij}=
   \dualactthree\left([A]_{ij}^{\Or},
   [B]_{ij}^{\Or}),2/(2r+1)\right)$.
   Let 
   $E_0=\{(i,j)\in\mathbb{Z}_n: A_{ij}\cdot B_{ij}=0
   \land \dualactthree_{ij}\ne 0\}$ 
   and
   $E_1=\{(i,j)\in\mathbb{Z}_n: A_{ij}\cdot B_{ij}=1
   \land \dualactthree_{ij}\ne 1\}$.
   Since clearly $0\le\dualactthree_{ij}\le 1$ and $\inner{A}{B}=|\{(i,j):A_{ij}\cdot B_{ij}=1\}|$, we have
   \begin{align}
        \inner{A}{B}-|E_1|
       &\le \mathbb{E}_F\left[\sigma(F*A),\sigma(F*B)\right]
       \le \inner{A}{B}+|E_0|\;.
   \end{align}
   In particular, if we show $E_0, E_1\subseteq\partial_r(A, B)$,
   then~\eqref{eq:04} will be established. This is what we show
   in the rest of the proof.
   
   For readability, we repeat  here the two conditions for a pair $(i,j)$ to be included in the boundary, see \defref{def:boundary}. Namely,
   $(i,j)$ belongs to $\partial_r(A,B)$ if and only
   if both conditions below hold:
   \begin{enumerate}
    \item$\exists a_1,b_1,c_1,d_1\in\{-r,\ldots,r\}:
        (i+a_1,j+b_1)\in A\land
        (i+c_1,j+d_1)\in B$.
    \item$\exists a_2,b_2,c_2,d_2\in\{-r,\ldots,r\}:
        (i+a_2,j+b_2)\notin A\lor
        (i+c_2,j+d_2)\notin B$.
%
%
    \end{enumerate}
   
   Let $(i,j)$ be such that
   $A_{ij} \cdot B_{ij}=0$. 
   By equation~\eqref{eq1}, $\dualactthree_{ij}=0$ if and only if $[A]^{\Or}_{ij}=0$ or $[B]^{\Or}_{ij}=0$.  So, when $A_{ij} \cdot B_{ij}=0$, for $\dualactthree_{ij}\ne 0$ to hold
   it is necessary that
there are  $ a_1, b_1, c_1, d_1 \in\{-r,\ldots,r\}$ such that $A_{i+a_1,j+b_1}=1$ and $A_{i+c_1,j+d_1}=1$ which corresponds to the first item in  \defref{def:boundary}. The second item holds with $a_2=b_2=c_2=d_2=0$ since $A_{ij} \cdot B_{ij}=0$ so either $A_{ij}=0$ or $B_{ij}=0$. This shows  that pairs of indices 
   for which $A_{ij}\cdot B_{ij}=0$ and
   $\dualactthree_{ij}\ne 0$ are contained  in  $\partial_r(A , B)$,
   in other words that $E_0\subseteq\partial_r(A,B)$.
   
   
   Now take $(i,j)$ such that $A_{ij} \cdot B_{ij}=1$. 
   By equation~\eqref{eq1},  $\dualactthree_{ij}=1$ if and only if $[A]^{\Or}_{ij}=\mathbf{1}_{(2r+1)\times (2r+1)}$ and $[B]^{\Or}_{ij}=\mathbf{1}_{(2r+1)\times (2r+1)}$. So, when $A_{ij} \cdot B_{ij}=1$ and $\dualactthree_{ij}\ne 1$, then
   there are  $ a_2, b_2, c_2, d_2 \in \{-r,\ldots,r\}$ such that $A_{i+a_2,  j+b_2}=0$ or $B_{i+c_2,  j+d_2}=0$ which corresponds to the second item in \defref{def:boundary}. The first item holds with $a_1=b_1=c_1=d_1=0$ since $A_{ij} \cdot B_{ij}=1$ so both $A_{ij}=1$ and $B_{ij}=1$. This shows that pairs of indices 
   for which $A_{ij}\cdot B_{ij}=1$ and $\dualactthree_{ij}\ne 1$ are
   contained in $\partial_r(A,B)$, or equivalently that
   $E_1\subseteq \partial_r(A,B)$.
   
   This concludes the proof of~\eqref{eq:04}. As
   for~\eqref{eq:09}, it follows easily by substituting 
   in the formula $\|A-B\|^2=\|A\|^2+\|B\|^2-2\inner{A}{B}$.
   We then use the fact
   $\E{ \left\| \ReLU(F*A)^2 \right\|}=\|A\|^2$, which follows
   from~\eqref{eq:08}, recalling that $\nu^2=2/(2r+1)^2$ and
   $\nu_{\ReLU}^2=1/2$.
   \hfill \qedsymbol

\end{document}